%% file: main.tex
\documentclass{article}

\PassOptionsToPackage{allcolors=mydarkblue,colorlinks=true}{hyperref}
\PassOptionsToPackage{table}{xcolor}

\PassOptionsToPackage{numbers, compress}{natbib}
 \usepackage[preprint]{neurips_2026}

\usepackage[utf8]{inputenc} % allow utf-8 input
\usepackage[T1]{fontenc}    % use 8-bit T1 fonts
\usepackage{hyperref}       % hyperlinks
\usepackage{url}            % simple URL typesetting
\usepackage{booktabs}       % professional-quality tables
\usepackage{amsfonts}       % blackboard math symbols
\usepackage{nicefrac}       % compact symbols for 1/2, etc.
\usepackage{microtype}      % microtypography
\usepackage{xcolor}         % colors

\usepackage{graphicx}
\usepackage{subcaption}
\usepackage{amsmath}
\usepackage{amssymb}
\usepackage{derivative}
\usepackage[makeroom]{cancel}
\usepackage{bm}
\usepackage{algorithm}
\usepackage{wrapfig}
\usepackage{enumitem}
\usepackage{tabularx}
\usepackage[capitalize,noabbrev,nameinlink]{cleveref}
\usepackage{tikz}
\usepackage{amsthm}
\usepackage{thmtools}
\usepackage{thm-restate}
\usepackage{pgfplots}
\usepackage{tikz-3dplot}
\usepackage[textsize=tiny]{todonotes}
\usepackage{algorithmic}
\usepackage{mathtools}

\usepackage{etoolbox}
\BeforeBeginEnvironment{wrapfigure}{\setlength{\intextsep}{0pt}}
\usepackage{varwidth}

\definecolor{mydarkblue}{rgb}{0,0.08,0.45} % Customize the RGB values as needed

\input{math_commands.tex}

\crefname{appendix}{App.}{Apps.}
\crefname{figure}{Fig.}{Figs.}
\crefname{equation}{Eq.}{Eqs.}
\crefname{section}{Sec.}{Secs.}
\crefname{algorithm}{Alg.}{Algs.}

\usetikzlibrary{arrows.meta,decorations.markings,calc}
\usetikzlibrary{angles,quotes}
\usetikzlibrary{decorations.text}
\pgfplotsset{compat=1.18}

\usetikzlibrary{positioning, shapes.geometric, fit, backgrounds}
\usetikzlibrary{patterns}

\newenvironment{talign*}
    {\csname align*\endcsname}
    {\endalign}
\newenvironment{talign}
    {\csname align\endcsname}
    {\endalign}
\crefname{talign}{}{}
\crefname{equation}{}{}

\renewcommand{\mid}{\,|\,}

\renewcommand{\paragraph}[1]{\noindent\textbf{#1}~~}

\theoremstyle{plain}

\theoremstyle{definition}

\theoremstyle{remark}

\title{Softly Constrained Denoisers for Diffusion Models Applied to Partial Differential Equations}

\author{%
Victor M.\ Yeom-Song$^{1,2}$ \quad Severi Rissanen$^{1,2}$ 
\\ \textbf{Arno Solin}$^{1,2}$ \quad \textbf{Samuel Kaski}$^{1,2,3}$ \quad \textbf{Mingfei Sun}$^{3}$\\
$^1$ELLIS Institute Finland \quad $^2$Aalto University \quad $^3$University of Manchester\\
\texttt{\{victor.yeomsong,severi.rissanen,arno.solin,samuel.kaski\}@aalto.fi}\\
\texttt{mingfei.sun@manchester.ac.uk}
}

\begin{document}

\maketitle
\vspace{-2em}
\begin{abstract}
  Diffusion models have become a powerful generative prior for solutions of partial differential equations (PDEs). Existing approaches enforce physical constraints either by adding the PDE residuals as loss regularizers or through inference-time adjustments. These methods bias the model away from the true data distribution, which is especially problematic when the governing PDE is misspecified. To circumvent these issues while making the most out of the PDE constraint, we introduce soft inductive biases into the denoiser architecture derived from the PDEs. We show that these \emph{softly constrained denoisers} exploit constraint knowledge to improve compliance over standard denoisers, while maintaining enough flexibility to deviate from it in case of misspecification with respect to observed data.\looseness-3
\end{abstract}

\section{Introduction}

Diffusion models \citep{ho_denoising_2020, song_score-based_2021} have become a powerful tool for generating solutions of partial differential equations (PDEs), with applications in fluid dynamics, wave propagation, and material science \citep{jacobsen_cocogen_2025, bastek_physics-informed_2025, huang2024diffusionpde, yao2025guided}. A central challenge in these applications is incorporating information from the governing PDE into the generative model. This is made harder by the fact that PDE models are inherently approximate: Darcy's law assumes laminar steady-state flow, the Helmholtz equation depends on accurate wave number measurements, and more broadly, any mathematical model involves simplifying assumptions that may not hold exactly for the observed data \citep{finzi_soft_2021, zou_correctingpinn_2024}.

\citet{finzi_soft_2021} tackle constraint misspecification by equipping neural networks with \emph{soft inductive biases}: structural preferences toward constraint-compliant solutions that do not restrict the hypothesis space. Such networks exploit approximate mathematical models where they are informative, while remaining free to follow the data where the model is wrong (\cref{fig:toySetup}).

\begin{figure}[h!]

% ------- Parameters -------
\def\eps{0.25}        % dent depth (bigger = deeper)
\def\Delta{1.20}      % angular half-width (radians)
\def\thetaZero{pi/2}  % dent center angle (radians); pi/2 = top
\def\cc{0.60}         % shoulder shaping (0..~0.8)
\def\kk{5}            % polynomial bump power k (>=4 recommended)

% ------- Helper functions (radians throughout) -------
\pgfmathdeclarefunction{wrap}{1}{% map angle to (-pi,pi]
  \pgfmathparse{mod(#1 + pi, 2*pi) - pi}%
}
% Compact-support polynomial bump: beta(u)=(1-u^2)^k for |u|<1, else 0
\pgfmathdeclarefunction{pbump}{2}{% pbump(u,k)
  \pgfmathparse{abs(#1) < 1 ? (1 - #1*#1)^#2 : 0}%
}
% Radius r(θ): exactly 1 outside the selected arc
\pgfmathdeclarefunction{rr}{1}{%
  \pgfmathsetmacro{\u}{wrap(#1-\thetaZero)/\Delta}%
  \pgfmathsetmacro{\B}{pbump(\u,\kk)}%
  \pgfmathparse{1 - \eps * \B * (1 + \cc*(1 - 2*\u*\u))}%
}

% \begin{subfigure}[b]{.48\columnwidth}%{.245\textwidth}
\begin{subfigure}[t]{.245\textwidth}
\centering
\begin{tikzpicture}
\begin{axis}[
  axis equal image, hide axis,
  xmin=-1.3, xmax=1.3, ymin=-1.3, ymax=1.3,
  width=.95\textwidth,
  height=.95\textwidth,  
  scale only axis
]

% Bulged circle
\addplot [domain=0:2*pi, samples=1200, line width=1.5pt, red!40]
  ({ rr(x) * cos(deg(x)) }, { rr(x) * sin(deg(x)) });

% Reference unit circle
\addplot [domain=0:2*pi, samples=500, dashed]
  ({cos(deg(x))}, {sin(deg(x))});

% Top label (outside, centered on upper arc)
\addplot [
  domain=pi/6:5*pi/6, samples=200, draw=none,
  postaction={decorate},
  decoration={
    text along path,
    reverse path,                 % make it read left->right on the top arc
    text align={center},
    raise=1.0ex,                  % push text outward from the circle
    text={|\scriptsize|Simplified constraint}
  }
] ({cos(deg(x))}, {sin(deg(x))});

% Bottom label
\node[font=\scriptsize\bf,color=red!40,align=center] at (axis cs: 0, -.5) {True data \\ distribution};

\end{axis}
\end{tikzpicture}
\caption{Problem setup} \label{fig:toySetup}
\end{subfigure}
\begin{subfigure}[t]{.245\textwidth}
  \centering
  \includegraphics[width=.85\linewidth,trim=0 15 0 15,clip]{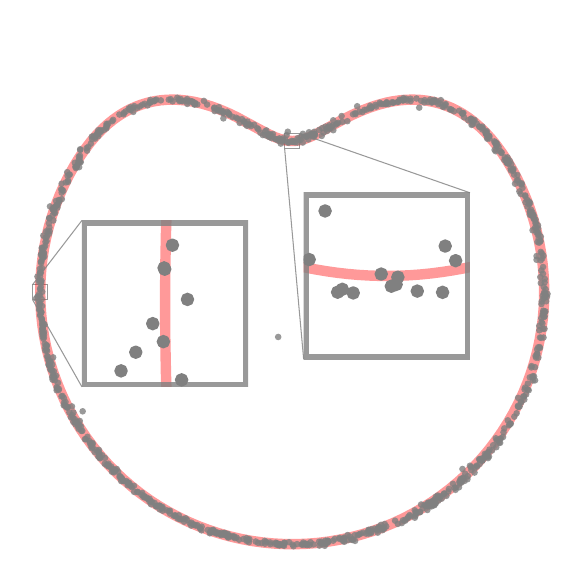}
  \caption{Vanilla Diffusion}
\end{subfigure}
\begin{subfigure}[t]{.245\textwidth}
  \centering
  \includegraphics[width=.85\linewidth,trim=0 15 0 15,clip]{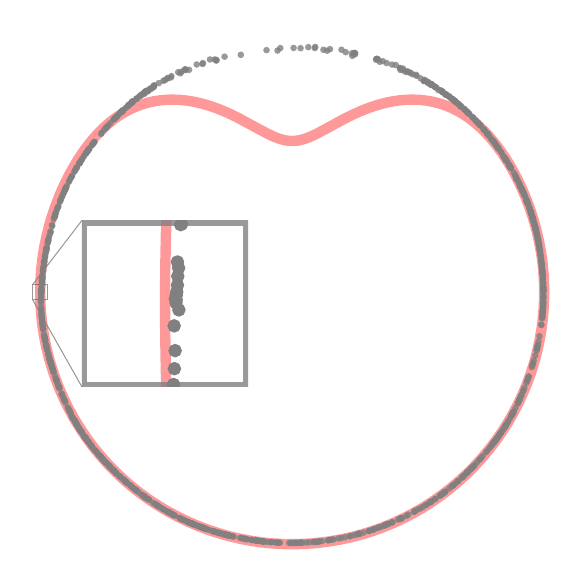}
  \caption{\centering Regularizer baseline} \label{fig:toyPIDM}%Physics-Informed Diffusion Models (regularizer baseline)} \label{fig:toyPIDM}
\end{subfigure}
\begin{subfigure}[t]{.245\textwidth}
  \centering
  \includegraphics[width=.85\linewidth,trim=0 15 0 15,clip]{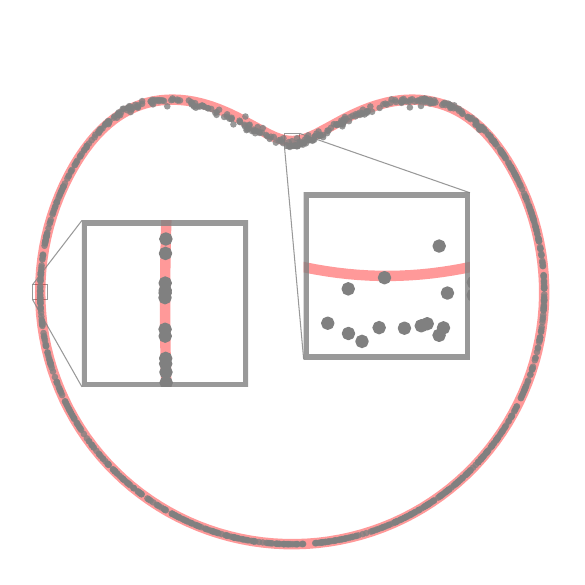}
  \caption{\centering Softly Constrained Denoisers (ours)} \label{fig:toyOurs}
\end{subfigure}
\caption{
Mathematical constraints, even if partially inaccurate, are useful to \textbf{learn an underlying data distribution} (a). Purely data-driven learning may miss the true distribution (b). Regularizer-based methods may overcommit to the constraint at the expense of data fidelity (c). Our method (d) exploits the constraint where it is informative, producing samples that align perfectly with the true data distribution (left), and defaults to vanilla behavior where it is not (top).}
\label{fig:dentfig1} \vspace{-2em}
\end{figure}

\begin{wrapfigure}{r}{.4\linewidth}
\raggedleft\scriptsize
% \includegraphics[width=0.5\linewidth]{}
% \caption{Caption}
\tdplotsetmaincoords{70}{110} % Set viewing angle (theta, phi)
\begin{tikzpicture}[tdplot_main_coords, scale=1.5]
    % Define axis length
    \def\axislength{3}
    \def\squaresize{0.5} % Size of the squares at origin
    
    % % Draw squares at origin to show orthogonality between each pair of axes
    % % Square in xy-plane (perpendicular to z-axis)
    % \draw[thick, fill=blue!15, opacity=0.6] 
    %     (0,0,0) -- (\squaresize,0,0) -- (\squaresize,\squaresize,0) -- (0,\squaresize,0) -- cycle;
    % % Square in xz-plane (perpendicular to y-axis)
    % \draw[thick, fill=red!15, opacity=0.6] 
    %     (0,0,0) -- (\squaresize,0,0) -- (\squaresize,0,\squaresize) -- (0,0,\squaresize) -- cycle;
    % % Square in yz-plane (perpendicular to x-axis)
    % \draw[thick, fill=green!15, opacity=0.6] 
    %     (0,0,0) -- (0,\squaresize,0) -- (0,\squaresize,\squaresize) -- (0,0,\squaresize) -- cycle;

    % Draw squares at origin to show orthogonality between each pair of axes
    % Using grayscale with different patterns to communicate 3D-ness
    
    % Square in xy-plane (perpendicular to z-axis) - horizontal lines
    \draw[thick, pattern=north west lines, pattern color=gray!70] 
        (0,0,0) -- (\squaresize,0,0) -- (\squaresize,\squaresize,0) -- (0,\squaresize,0) -- cycle;
    
    % Square in xz-plane (perpendicular to y-axis) - vertical lines
    \draw[thick, pattern=north east lines, pattern color=gray!70] 
        (0,0,0) -- (\squaresize,0,0) -- (\squaresize,0,\squaresize) -- (0,0,\squaresize) -- cycle;
    
    % Square in yz-plane (perpendicular to x-axis) - diagonal lines
    \draw[thick, pattern=crosshatch, pattern color=gray!70] 
        (0,0,0) -- (0,\squaresize,0) -- (0,\squaresize,\squaresize) -- (0,0,\squaresize) -- cycle;
    
    % Draw axes
    \draw[thick, ->] (0,0,0) -- (\axislength,0,0);
    \draw[thick, ->] (0,0,0) -- (0,2,0);
    \draw[thick, ->] (0,0,0) -- (0,0,1.5);

    % Axis labels
    \node[anchor=north east] at (2.5,1.5,0) {Denoiser architecture};
    \node[anchor=north west] at (0,1.75,-.1) {Training};
    \node[anchor=south] at (0,0,1.5) {Sampler};
    
    % Add blue dots with labels on x and y axes
    \fill[red!70!black] (2,0,0) circle (3pt) node[anchor=west, yshift=0pt, xshift=8pt, text=red!70!black]{This work};
    \fill[blue!70!black] (0,1.5,0) circle (3pt) node[anchor=south, yshift=8pt, xshift=0pt, text=blue!70!black]{Regularizers};
    % Add red dot with label on z axis
    \fill[green!70!black] (0,0,1.) circle (3pt) node[anchor=west, yshift=0pt, xshift=8pt, text=green!70!black]{Inference time adjustments};
    
    % Optional: Add grid on xy-plane
    % \draw[gray!30, very thin] (0,0,0) grid (\axislength,\axislength);
    
\end{tikzpicture}
    \caption{Previous works have focused on the training \cite{bastek_physics-informed_2025} and sampling \cite{huang2024diffusionpde} axes of the design space to enforce constraints. We present an orthogonal approach to enforce soft constraints through the denoiser architecture.}
    \label{fig:contributionDiag}
\end{wrapfigure}
We bring this idea to the diffusion model setting, where it has not previously been explored. The main challenge in doing so is that diffusion models have a specific structure: the denoiser must approximate the conditional mean to guarantee correct score estimation, and it is not obvious how to introduce a soft bias without breaking this connection. We derive a principled answer from the guidance literature: a correction term based on the constraint gradient, scaled by a learned factor that preserves the denoising score matching optimum while providing a structural shortcut toward constraint-compliant outputs. We call the resulting architecture the \emph{Softly Constrained Denoiser} (SCD). Unlike regularizer-based approaches \citep{bastek_physics-informed_2025}, which bias the generative distribution away from the data (as we prove in \cref{prop:elbo_increase} and illustrate in \cref{fig:toyPIDM}), and unlike inference-time guidance \citep{jacobsen_cocogen_2025, huang2024diffusionpde}, which introduces approximation errors during sampling, SCD operates on the architecture axis of the diffusion model design space (\cref{fig:contributionDiag}): the training loss and sampling procedure remain unchanged, so the model retains the flexibility to follow the data where the constraint is wrong (\cref{fig:toyOurs}).

To summarize, the contributions of this paper are:
\begin{itemize}[leftmargin=2em]
     \item We propose a method to transform any denoiser to a ``softly constrained denoiser'' with a soft inductive bias towards constraint-compliant samples with minimal computational overhead.
     \item We prove that previous regularizer methods break the distribution-modeling guarantees of standard diffusion models, and empirically show that these methods are particularly sensitive to constraint misspecification. Critically, we show that SCDs preserve these guarantees.
     \item We empirically demonstrate the effectiveness of our approach through illustrative problems and PDE benchmarks in misspecified settings, showing superior constraint satisfaction over existing methods, while maintaining sample quality and robustness to misspecification.
\end{itemize}

\section{Background}
\label{section:background}

Diffusion models generate samples from a data distribution $p(\bm{x}_{0})$ by learning how to denoise samples from a forward \emph{noising} process \citep{sohl-dickstein_deep_2015, ho_denoising_2020, song_score-based_2021, karras_elucidating_2022}, which is generally assumed to be of the form:
\begin{talign}
    p(\bm{x}_{t}) = \int \mathcal{N}\left( \bm{x}_{t}; \bm{x}_{0}, \sigma(t)^2 \mathbf{I} \right) p(\bm{x}_{0}) \odif{\bm{x}_{0}}.
\end{talign}
In simpler terms, clean samples $\bm{x}_{0}$ from the data distribution $p(\bm{x}_{0})=p_{\text{data}}(\bm{x}_{0})$ are corrupted by a Gaussian process $\mathcal{N}(\bm{0}, \sigma(t)^{2}\mathbf{I} )$ at time $t$. The corresponding reverse \textit{denoising} process can be formulated as a probability flow ODE \citep{karras_elucidating_2022}:
\begin{talign}
    \odif{\bm{x}_{t}} = -\dot{\sigma}(t)\sigma(t)\nabla_{\bm{x}}\log p(\bm{x}_{t}) \odif{t}.\label{eq:prob_flo_ode}
\end{talign}
Starting with a sample from an isotropic Gaussian $\mathcal{N} (\bm{x}_{t}; \bm{x}_{0}, \sigma_{\text{max}}^{2}\mathbf{I})$ and integrating the ODE backwards in time, it is possible to recover a sample from the original data distribution $\bm{x} \sim p(\bm{x}_{0})$, as long as the score is learned accurately and $\sigma_{\text{max}}$ is large enough \citep{song_score-based_2021}.
To get $\nabla_{\bm x_t}\log p(\bm x_t)$, we learn a denoiser conditional on noise level $t$ and corrupted samples $\bm x_t$ of training data $\bm x_0$:
\begin{talign}
    \mathcal{L} (\theta) = \mathbb{E}_{t\sim p(t), \bm x_{0}, \bm{x}_{t}} \left[w(t)\|D_\theta(\bm x_t, t) - \bm x_0\|^{2} \right], \label{eq:dsm}
\end{talign}
where $w(t)$ and $p(t)$ define the weighting and sampling frequency of noise levels during training, and $D_{\theta}$ is the diffusion model's denoiser with parameters $\theta$. At convergence, $D_\theta(\bm x_t, t) \approx \mathbb{E}[\bm x_0 \mid \bm x_t]$. Combined with Tweedie's formula $\mathbb{E}[\bm x_0 \mid \bm x_t] = \bm x_t + \sigma(t)^2 \nabla_{\bm x_t} \log p(\bm x_t)$, this ensures that we can recover an approximation of the score $\nabla_{\bm x_t}\log p(\bm x_t)$ with $s_\theta(\bm x_t) = \frac{D_\theta(\bm x_t, t) - \bm x_t}{\sigma(t)^2}$. Accordingly, the loss in  \cref{eq:dsm} is also called the \emph{denoising score matching} loss~\citep{vincent_connection_2011, song_score-based_2021}.

\paragraph{Distributional bias}
The core problem in generative modeling is to learn a surrogate distribution $p_{\theta}(\bm{x})$ parameterized by $\theta$ to approximate a data distribution $p_{\text{data}}(\bm{x})$ \citep{tomczak2022deep}. We call a generative framework \emph{biased} if $p_{\theta}(\bm{x})$ does not converge to $p_{\text{data}}(\bm{x})$ under optimal conditions, i.e., after finding the global optimum of the loss with infinite data, the sampling procedure does not result in samples from $p_{\text{data}}(\bm{x})$. Importantly, the diffusion model training and sampling in \cref{eq:dsm} and \cref{eq:prob_flo_ode} is unbiased in this sense and can thus approximate any data distribution. 

\paragraph{Constraint misspecification}
In science, we build simplified mathematical models around complex phenomena and systems to facilitate their study and usage. However, since these models are inherently approximations to real phenomena, they are bound to have varying degrees of \emph{misspecification}. Formally, a perfectly specified constraint $c^{*}(\bm x)\in \{0,1\}$, where $\mathcal{X}_{c^{*}} = \{\bm x: c^{*}(\bm x)= 1\}$ is the constraint-satisfying set, has the property that for any $\bm x$, $p_{\text{data}}(\bm x)c^{*}(\bm x) = p_{\text{data}}(\bm x)$. However, often we do not have access to $c^{*}(\bm x)$: we derive some approximation $\tilde{c}(\bm x)$ based on the system under study and limited samples from $\mathcal{X}_{c^{*}}$. Our goal is to exploit $\tilde{c}(\bm x)$ to gather the most information possible about $\mathcal{X}_{c^{*}}$ and lead the model towards the best approximation to $p_{\text{data}}$ under limited resources. Misspecification is an issue especially when the generative framework utilizing the constraint is prone to bias, i.e., it does not have guarantees for converging to the data distribution $p_{\text{data}}$ at some limit. Furthermore, an incorrect constraint could push the model even further away from $p_{\text{data}}$.

\paragraph{Guided generation}
Assume we have a \emph{constraint function} $c(\bm x)$ where the constraint is satisfied when $c(\bm x)=1$ and not satisfied otherwise.
It could be a hard constraint $c(\bm x) \in \{0,1\}$, or a relaxed continuous constraint $c(\bm x) \in [0,1]$. We can turn a trained diffusion model with output distribution $p(\bm x_0)$ into a constrained model with distribution $p(\bm x_0) c(\bm x_0)$ by adjusting the score as follows \citep{song_pseudoinverse-guided_2023, chung_diffusion_2023}:
\begin{talign}
    s_{\text{adjusted}}(\bm x_t) = \nabla_{\bm x_t}\log p(\bm x_t) + \nabla_{\bm x_t} \log \int c(\bm x_0) p(\bm x_0 \mid \bm x_t) \,\text{d}\bm x_0 . \label{eq:guidance_adjustment}
\end{talign}
Many \emph{inference-time} methods, methods that leave model weights unchanged, use the second term on the right-hand side to modify the sampling trajectory \citep{ho_video_2022, song_pseudoinverse-guided_2023, chung_diffusion_2023, song_loss-guided_2023, rissanen_free_2025}. To use this term, we must choose an approximation for $p(\bm{x}_{0}\mid\bm{x}_{t})$ and a scheme to evaluate the integral. In combination with inference-time adjustments, any approximation errors may bias the output distribution.

\paragraph{Regularization}
Another approach to constrain the generative space of a model is to use \emph{regularizers} on the training objective. \citet{bastek_physics-informed_2025} introduce a residual-based regularizer to the standard DDPM training objective \citep{ho_denoising_2020}, resulting in the following expression:
\begin{talign}
    \mathcal{L}_{\text{target}}\left( \hat{\vx}_{0}, t \right) = \mathcal{L}_{\text{DDPM}} \left( \hat{\vx}_{0} \right) + \lambda_{t} \lVert \mathcal{R}_{\text{constraint}}\left( \hat{\vx}_{0} \right) \rVert, \label{eqn:regularized_loss}
\end{talign}
where $\hat{\vx}_{0} = D_{\vtheta}(\vx_{t},t)$ is the output of the denoiser, $\lambda_{t} \geq 0$ is a hyperparameter that weighs the constraint compliance at denoising time step $t$, $\mathcal{R}$ is a residual used to evaluate the denoiser estimate, and $\lVert \cdot \rVert$ is some scalar norm of choice, e.g., $L_{p}$. Although the target data distribution should naturally minimize this residual, PDE-based regularizers can make the loss landscape hard to optimize \citep{Krishnapriyan_characterizing_2021, rathore_challenges_2024}. Further, the optimum of \cref{eqn:regularized_loss} forfeits the property that $D_\theta(\bm x_t, t) \approx \mathbb{E}[\bm x_0 \mid \bm x_t]$ at convergence and the connection between the denoiser and $\nabla_{\bm x_t} \log p(\bm x_t)$ is lost (see \cref{prop:shifted_optimum}). Thus, the addition of these targets in the loss function biases the generative distribution \citep{bastek_physics-informed_2025, baldan_flow_2025}.
\section{Methods}
\label{sec:methods}

We introduce a denoiser parameterization that exploits PDE knowledge to improve sample quality without sacrificing the standard diffusion model distributional guarantees. The denoiser is decomposed into a base network $D^{\text{orig}}_\theta(\bm{x}_t, t)$ that learns from data, and a constraint-informed correction derived from the PDE, controlled by a learned factor $\gamma_\theta(\bm{x}_t, t)$ that modulates how much the model relies on the constraint at each noise level. The combined denoiser (\cref{fig:scd_architecture}) is trained with the standard DSM loss, preserving the optimum $\mathbb{E}[\bm{x}_0 | \bm{x}_t]$. This satisfies two goals simultaneously: {\em (i)}~the model exploits constraint knowledge to improve physical plausibility, even under partial misspecification, and {\em (ii)}~it retains the guarantee of converging to $p_{\text{data}}$. Rather than imposing the constraint externally through a regularizer or guidance adjustments, the denoiser \emph{learns} when and how much to use it.

\begin{figure}[tbp]
    \centering
    % --- LEFT SIDE: ALGORITHM ---
    \begin{minipage}[t]{0.48\textwidth}
        \vspace{2pt} % Align tops
        % We use captionof to caption the algorithm inside a minipage properly
        % Decreasing font size to footnote size to ensure math fits in narrow column
        \footnotesize 
        \begin{algorithm}[H]
        \caption{Loss for SCD\label{alg:scd_training}}
        \begin{algorithmic}[1]
            \REQUIRE Data $p_{\text{data}}(\vx)$, constraint $l_c(\vx)$, noise $\sigma(t)$, weights $w(t)$, noise $p(t)$, network parameters $\vtheta$
            \STATE Sample $\vx_0, t, \vepsilon$
            \STATE $\vx_t \leftarrow \vx_0 + \sigma(t)\vepsilon$
            
            \STATE \textbf{Compute SCD Output $\hat{\vx}_{\vtheta}$:}
            \STATE $\hat{\vx}_{base} \leftarrow D_{\vtheta}^{orig}(\vx_t, t)$
            \STATE $g \leftarrow \nabla_{\vx} \log l_c(\hat{\vx}_{base})$
            \STATE $\hat{\eta} \leftarrow \gamma_{\vtheta}(\vx_t, t) \sigma(t)^2$
            % Split long line for narrow column
            \STATE $\hat{\vx}_{\vtheta} \leftarrow \hat{\vx}_{base} + \hat{\eta} \cdot g$ 
            
            \STATE \textbf{Output Loss:} 
            \STATE $\mathcal{L} \leftarrow w(t) \| \hat{\vx}_{\vtheta}(\vx_t, t) - \vx_0 \|^2$
        \end{algorithmic}
    \end{algorithm}
    \end{minipage}
    \hfill
    % --- RIGHT SIDE: FIGURE ---
    \begin{minipage}[t]{0.48\textwidth}
    \vspace{0pt} % Align tops
    \centering
    % Resize the TikZ picture to fit exactly within the minipage width
    \resizebox{\columnwidth}{!}{
    \begin{tikzpicture}[
        node distance=1.2cm and 1.2cm, % Slightly tighter spacing
        >={Latex[width=2mm,length=2mm]},
        block/.style={
            rectangle, 
            draw=black!80, 
            fill=blue!5, 
            thick, 
            minimum height=1cm, 
            minimum width=2.2cm, % Slightly narrower blocks
            align=center,
            rounded corners
        },
        op/.style={
            circle,
            draw=black!80,
            fill=orange!10,
            thick,
            minimum size=0.8cm,
            inner sep=0pt
        },
        param/.style={
            rectangle,
            draw=green!60!black,
            fill=green!5,
            dashed,
            minimum height=0.8cm,
            align=center,
            rounded corners
        },
        txt/.style={
            font=\small\sffamily
        }
    ]
    
    % --- Nodes ---
    \node (input) [txt, font=\large] {$\mathbf{x}_t, t$};
    
    % Base Denoiser Row
    \node (base) [block, right=0.8cm of input, anchor=west] {Base Denoiser\\$D_{\boldsymbol{\theta}}^{orig}(\mathbf{x}_t, t)$};
    \node (xbase) [right=0.8cm of base, txt] {$\hat{\mathbf{x}}_{base}$};
    \node (add) [op, right=0.8cm of xbase] {$+$};
    \node (output) [txt, right=0.6cm of add, font=\large] {$\hat{\mathbf{x}}_{\boldsymbol{\theta}}$};
    
    % Constraint Gradient
    \node (grad) [block, below=0.6cm of xbase, fill=red!5] {Constraint Gradient\\$\nabla_{\mathbf{x}} \log l_c(\cdot)$};
    
    % Scaling Network
    \node (gamma) [block, below=2.2cm of base] {Scaling Network\\$\gamma_{\boldsymbol{\theta}}(\mathbf{x}_t, t)$};
    
    % Multiplication Node
    \node (mult) [op] at (gamma -| add) {$\times$};
    
    % Noise Level
    \node (sigma) [param, below=0.6cm of mult] {Noise Level\\$\sigma(t)^2$};
    
    % --- Connections ---
    \draw[->] (input) -- (base);
    \draw[->] (input) |- (gamma);
    \draw[->] (base) -- (xbase);
    \draw[->] (xbase) -- (add);
    \draw[->] (add) -- (output);
    \draw[->] (xbase) -- (grad);
    \draw[->] (grad) -- node[midway, above right=-0.1cm, txt] {$g$} (mult);
    \draw[->] (gamma) -- node[midway, above, txt] {$\hat{\gamma}$} (mult);
    \draw[->] (sigma) -- (mult);
    \draw[->] (mult) -- node[midway, right] {Adjustment} (add);
    
    % --- Background Group ---
    \begin{scope}[on background layer]
        \node [
            fit=(base) (gamma) (grad) (mult) (sigma) (add), 
            fill=gray!5, 
            rounded corners, 
            draw=gray!20, 
            dashed,
            inner sep=0.3cm,
            %label={[anchor=south west, xshift=0.5cm, yshift=0.1cm]north west:\textit{\color{gray}Softly Constrained Denoiser Logic}}
        ] (group) {};
    \end{scope}
    \end{tikzpicture}%
    } % End resizebox
        \captionof{figure}{Architecture of SCD. Using a learnable $\gamma_\theta$, the network is able to modulate how much to exploit the constraint information.}
        \label{fig:scd_architecture}
    \end{minipage}
    \vspace{-1em}
\end{figure}
\paragraph{Class of constraints} We consider constraints that can be written as $\mathcal{R}(\vx)=0$, where $\mathcal{R}$ is a \emph{residual function}. Thus, $c(\vx)=1$ if $\vx \in \{\vx: \mathcal{R}(\vx)=0\}$ and $0$ otherwise. This constraint can be misspecified, but should be informative about the data distribution. We further take the standard assumption that $\mathcal{R}$ is continuously differentiable in the input $\vx$ \citep{huang2024diffusionpde,bastek_physics-informed_2025,yao2025guided}, so we can calculate gradients $\nabla_{\vx}\mathcal{R}(\vx)$. We then define a \emph{relaxed constraint function} $l_c(\vx)=\exp(-\lVert \mathcal{R}(\vx) \rVert)$ for some choice of norm $\lVert \cdot \rVert$, as prescribed by \citet{song_loss-guided_2023}. The norm and more broadly the design of $l_c(\vx)$ are design choices for our method, ideally designed to be maximally informative for the denoiser. 

\paragraph{Deriving the Softly Constrained Denoiser} Let us write the original denoiser output as $D_\theta^{\text{orig}}(\vx_t, t)$ (e.g., the output of a standard U-Net). The aim is to combine $l_c$ with $D_\theta^{\text{orig}}(\vx_t, t)$ such that it becomes easier for the combined network $D_\theta(\vx_t, t)=f(D_\theta^{\text{orig}}(\vx_t, t), l_c(\cdot ))$ to capture relevant details in the output to generating constraint-compliant samples. To achieve this, we ask: what kinds of non-learnable modifications to $D_\theta(\vx_t, t)$ tend to improve constraint satisfaction? We will use the answer to define our final, fully learnable denoiser. Luckily, this question has been extensively studied in the inference-time guidance literature, allowing us to build on well-known results there. 

\paragraph{Nudging $D_\theta(\vx_t, t)$ towards constraint satisfaction}
We integrate the soft inductive biases from the PDE residuals into the denoiser building up from ideas in the guidance literature: we first propose a practical approximation to \cref{eq:guidance_adjustment}. This leads to a formula that nudges the denoiser towards satisfying the constraint by using the gradient of a (relaxed) constraint function $\nabla_{\bm x_0}l_c(\bm x_0)$. We will use this formula as an inspiration for the final network design. 

We choose an approximation to $p(\bm x_0 \mid \bm x_t)$ in \cref{eq:guidance_adjustment}. A common choice is $p(\bm x_0 \mid \bm x_t) = \mathcal{N}( D_\theta(\bm x_t, t), \sigma_{0|t}^{2} \mathbf{I})$, where $\sigma_{0|t}^{2}$ is a hyperparameter \citep{ho_video_2022, song_pseudoinverse-guided_2023, chung_diffusion_2023, boys_tweedie_2024, peng_improving_2024, rissanen_free_2025}. Similarly to \citet{chung_diffusion_2023}, we choose $\sigma_{0|t}^{2} = 0$. Thus, $p(\bm{x}_{0} \mid \bm{x}_{t})$ turns into a Dirac delta on $D_{\theta}(\bm{x}_{t}, t)$ and taking the integral in \cref{eq:guidance_adjustment} with this choice produces the following:
\begin{talign}
\label{eqn:approxConditionalScore}
    s_{\theta}^{\text{guided}}\left( \bm{x}_{t}\right) \approx \frac{D_{\theta}(\bm{x}_{t}, t) - \bm{x}_{t}}{ \sigma_{t}^{2} } + \nabla_{\bm{x}_{t}} \log l_{c} \left( D_{\theta}(\bm{x}_{t}, t) \right).
\end{talign}
This results in a vector-Jacobian product through $D_{\theta}$, with considerable computational overhead:
\begin{talign}
    \nabla_{\bm{x}_{t}} \log l_{c} \left( D_{\theta}(\bm{x}_{t}, t)\right)^\top = \nabla_{D_\theta} \log l_{c} \left( D_{\theta}(\bm{x}_{t}, t)\right)^\top \nabla_{\bm{x}_{t}} D_{\theta}(\bm{x}_{t}, t).\label{eq:guidance_formula_with_jacobian}
\end{talign}
To obtain a more efficient approximation,  we recall that the Jacobian is connected to the denoising covariance through $\nabla_{\bm{x}_{t}} D_{\theta}(\bm{x}_{t}, t)=\frac{\text{Cov}[\bm x_0 \mid \bm x_t]}{\sigma(t)^2}$ \citep{boys_tweedie_2024}. Analogous to our earlier treatment of $p(\bm x_0\mid \bm x_t)$ in \cref{eqn:approxConditionalScore}, we approximate the conditional covariance as $\text{Cov}[\bm x_0 \mid \bm x_t] \approx \bm \Lambda_t$, where $\bm \Lambda_t$ is some diagonal matrix. Altogether, this yields the following denoiser correction:
\begin{talign}
    D_{\theta, \text{guided}}(\bm x_t, t) = D_\theta(\bm x_t, t) + \bm\Lambda_t \sigma(t)^2\nabla_{\bm D_\theta} \log l_c(D_\theta(\bm x_t, t)).\label{eq:guided_denoiser}
\end{talign}

\paragraph{Softly Constrained Denoiser} Since \cref{eq:guided_denoiser} tends to generate samples that satisfy the constraint for a suitably chosen $\bm \Lambda_t$, 
we can form a denoiser parameterization that uses the same structure for easier constraint compliance through a learned $\bm \Lambda_t$ term:
\begin{equation}
    D^{\text{SCD}}_\theta(\bm x_t, t) = D_{\theta}^\text{orig}(\bm x_t, t)
    + \bm \gamma_{\theta}(\bm x_t, t)\sigma(t)^2 \nabla_{\bm D_\theta^\text{orig}} \log l_c(D_{\theta}^\text{orig}(\bm x_t, t)), \label{eq:constDenoiser}
\end{equation}
where $D_{\theta}^{\text{orig}}(\bm x_t, t)$ is the original denoiser output and $\bm{\gamma}_{\theta}(\bm x_t, t)$ is a learnable scaling factor. Details on the parameterization of $\bm\gamma_{\theta}$ are given in \cref{app:training}. The correction term in \cref{eq:constDenoiser} only evaluates the gradient of the constraint $l_{c}$ until $D^{\text{orig}}_{\theta}$, avoiding a costly vector-Jacobian product. In particular, stationarity conditions for the denoising score matching loss with respect to $D^{\text{SCD}}_\theta(\bm x_t, t)$'s parameterization show that the connection to the optimal denoiser is preserved:
\begin{restatable}{proposition}{PropOptimalDenoiserSCD}
    \label{prop:scd_optimal_denoiser}
    Given $D^{\text{SCD}}_\theta(\bm x_t, t)$, the stationarity conditions for the denoising score matching loss with respect to $D_{\theta}^{\text{orig}}$ and $\bm\gamma_{\theta}$ are the same, and connect $D^{\text{SCD}}_\theta(\bm x_t, t)$ to the optimal denoiser:
    \begin{talign}
        D^{\text{SCD}}_\theta = D_{\theta}^{\text{orig}}+\bm\gamma_{\theta}\sigma(t)^{2}\nabla_{D_{\theta}^{\text{orig}}}\log l_c(D_{\theta}^{\text{orig}})=\mathbb{E}[\bm x_{0}| \bm x_{t}] = \hat{\bm x}_{0}^{*}
    \end{talign}
\end{restatable}\vspace{-0.5em}
\begin{restatable}{corollary}{CorolGammaRole}
    \label{prop:gamma_role}
    The term $\bm\gamma_{\theta}\sigma(t)^{2}\nabla_{D_{\theta}^{\text{orig}}}\log l_c(D_{\theta}^{\text{orig}})$ encodes how much of the deviation of $D_{\theta}^{\text{orig}}$ from the optimal denoiser $\hat{\bm x}_{0}^{*}$ is explained by the constraint's gradient. At optimality:
    \begin{talign}
        \bm\gamma_{\theta}\sigma(t)^{2}\nabla_{D_{\theta}^{\text{orig}}}\log l_{c}(D_{\theta}^{\text{orig}})=\text{proj}_{\nabla_{D_{\theta}^{\text{orig}}}\log l_{c}(D_{\theta}^{\text{orig}})}(\hat{\bm x}^{*}_{0}-D_{\theta}^{\text{orig}})
    \end{talign}
\end{restatable}\vspace{-0.5em}
Proofs for these results are in \cref{app:proofs}. While we adopt this particular approximation, many alternative formulations of the guidance formula are possible, each defining a corresponding SCD. We show the training algorithm explicitly in \cref{alg:scd_training} and visualize the new denoiser architecture in \cref{fig:scd_architecture}.

\subsection{Analysis of regularization bias and ELBO degradation}

We formally analyze the impact of introducing a regularizer (e.g. PIDM \citep{bastek_physics-informed_2025}) to the training objective. We demonstrate that while regularization enforces constraint satisfaction, it biases the denoiser away from the true conditional expectation and also necessarily degrades the variational lower bound (ELBO) of the diffusion model \citep{ho_denoising_2020, kingma2021variational}.

Consider the regularized objective function for a specific noise level $t$ with constraint residual $\mathcal{R}$:
\begin{talign}
    \mathcal{L}_{\text{reg}}(\theta) = \mathbb{E}_{\vx_0, \vx_t} \left[ \| D_\theta(\vx_t, t) - \vx_0 \|^2 \right] + \lambda \| \mathcal{R}(D_\theta(\vx_t, t)) \|^2,
\end{talign}
where $\lambda > 0$ weighs the regularizer. In \cref{app:proofs}, we prove the following two propositions:

\begin{restatable}{proposition}{PropShiftedOptimum}
    \label{prop:shifted_optimum}
    Let $D^*_{\text{reg}}(\vx_t, t)$ be the denoiser that minimizes the regularized objective $\mathcal{L}_{\text{reg}}$. The optimal denoiser output is shifted as follows
    \begin{talign}
        D^*_{\text{reg}}(\vx_t, t) = \mathbb{E}[\vx_0 | \vx_t] - \lambda \left[ \nabla_y \mathcal{R}(y) \right]^\top \mathcal{R}(y) \Big|_{y=D^*_{\text{reg}}},
    \end{talign}
    so the optimal denoiser output $D_{\text{reg}}^*(x_t, t)$ is shifted from the conditional mean until the equation holds. Since $\mathbb{E}[x_0\mid x_t]$ does not generally satisfy the constraint, the optimal denoiser output is shifted. 
\end{restatable}
The asymptotic distributional guarantees of diffusion models are thus lost since there is no connection between the optimal denoiser and the score function $\nabla_{\vx_t}\log p(\vx_t)$.

\begin{restatable}{proposition}{PropIncreasedELBO}
    \label{prop:elbo_increase}
    $D^*_{\text{reg}}$ strictly degrades the Evidence Lower Bound (ELBO) compared to the vanilla denoiser $D^*_{\text{vanilla}}$. That is, the model's approximation of the data likelihood deteriorates.
\end{restatable}

\section{Related work}

\citet{karras_elucidating_2022} split the diffusion model design space into the denoiser architecture, training dynamics, and sampling algorithms. Most of the related work in ``diffusion models for PDEs'' has focused on the last two spaces, which we detail here. An extended related work is available in \cref{app:extRelatedWork}.

\paragraph{Diffusion models applied to PDEs } 
\citet{jacobsen_cocogen_2025} propose conditional PDE generation using a Controlnet-like conditioning structure \citep{zhang2023adding} and an inference-time adjustment where the final samples are optimized to have a small PDE residual. 
\citet{bastek_physics-informed_2025} present Physics-Informed Diffusion Models (PIDM), a framework to train DDPM-based diffusion models with a PDE regularizer term to minimize along the loss function.
Several works utilize DPS-like guidance \citep{chung_diffusion_2023} for PDE data assimilation, targeting noisy measurements \citep{shysheya2024conditional}, constraint satisfaction \citep{huang2024diffusionpde}, or Banach space diffusion models \citep{yao2025guided}. Similarly, \citet{chenggradient} employ projection-based sampling \citep{zhu_plugnplayrestore_2023}. Unlike these, our method avoids approximate inference-time guidance. Furthermore, our parameterization is orthogonal to recent architecture-focused works on neural operators \citep{huwavelet, oommen2024integrating} or GNNs \citep{valencia2025learning}, as it remains compatible with any base architecture.

\paragraph{Injecting measurement structure for training inverse problem solvers} Mathematically, the closest work is the likelihood-informed Doob's h-transform by \citet{denker2024deft}, who finetune adapters using observation gradients $\nabla_{\bm x_0} p(\bm y\mid \bm x_0)$, similar to our constraint-informed parameterization using $\nabla_{\bm x_0} l_c(\bm x_0)$. However, our motivation and methodology differ: they use likelihoods $p(\bm y\mid \bm x_0)$ for Bayesian inference from noisy observations, whereas we \emph{define} $l_c(\bm x_0)$ to restrict generation to a constrained subset. Their goal is an alternative to inference-time adjustment, while we seek to alleviate distributional biases and the model's sensitivity to constraint misspecification. Furthermore, we train from scratch, whereas they finetune adapters on larger models.

\paragraph{Inference time adjustment for inverse problems} Many methods adjust diffusion models at inference time to solve inverse problems, formally targeting approximations of \cref{eq:guidance_adjustment}. The first methods to make the explicit connection to \cref{eq:guidance_adjustment} were \citet{ho_video_2022, chung_diffusion_2023, song_pseudoinverse-guided_2023}. While the method by \citet{song_pseudoinverse-guided_2023} only worked for linear inverse problems, \citet{song_loss-guided_2023} generalized it to general guidance functions through Monte Carlo integration of \cref{eq:guidance_adjustment}.

\paragraph{Hard-constraint diffusion via modified dynamics} Several works impose constraints by \emph{changing the diffusion dynamics} so that the support is restricted by construction. Riemannian diffusion and flow matching models move the noising and denoising processes to a target manifold, enabling sampling on spheres, tori, hyperboloids, and matrix groups but requiring smooth geometry and geometric operators \citep{de2022riemannian,huang2022riemannian,chenflow}. \citet{fishmandiffusion, fishman2023metropolis, lou_reflected_2023} propose noising processes that are constrained within convex sets defined by inequality constraints, while \citet{liu_mirror_2023} propose learning standard diffusion models in a dual space created using a mirror map. These methods provide hard constraints but are typically specialized to particular geometries or constraint classes.

\paragraph{Optimizing samples to match with constraints} \citet{benhamu_dflow} generate samples with constraints by optimizing a source point in the noisy latent space such that the generative ODE solution matches with the constraint. \citet{tang2024tuning} instead optimize the noise injected during the stochastic sampling process. \citet{pooledreamfusion} generate samples by directly optimizing the target image within a constrained space, while minimizing the diffusion loss for the image. \citet{wang2023prolificdreamer} extend the method with a particle-based variational framework. In a similar manner, \citet{mardanivariational} formulate the sampling process as optimizing a variational inference distribution on the clean samples. \looseness-2

\paragraph{Soft inductive biases}
\citet{finzi_soft_2021} propose using ``dual path'' layers to build a neural network, where one path uses a hard-constrained layer and the other uses a more ``relaxed'' layer. By assigning a lower prior probability to the relaxed path, a soft inductive bias towards constraint-compliant solutions is imposed, without restricting the possible hypothesis space for the neural network. \looseness-2 
\section{Experiments}
\label{sec:experiments}

\begin{wrapfigure}{r}{.5\textwidth}
  \centering\footnotesize
    % \includegraphics[width=\linewidth]{figures/circles/circles_misspec_wass-1.pdf}
  % YOU CAN ALSO JUST HAVE AN EXRTERNAL TEX FILE AND USE \inpup{filename.tex}
    \begin{tikzpicture}[remember picture]
\begin{axis}[
    width=.5\columnwidth,
    height=.4\columnwidth,    
    xlabel={Angle of misspecification},
    ylabel={Average $W_1$ distance},
    xmin=-0.2, xmax=3.2,
    ymode=log,
    ymin=0.002, ymax=2,
    xtick={0, 1, 2, 3},
    xticklabels={$0$, $\frac{\pi}{4}$, $\frac{\pi}{3}$, $\frac{\pi}{2}$},
    ytick={0.01, 0.1, 1},
    yticklabels={$10^{-2}$, $10^{-1}$, $1$},
    legend style={
      at={(0.95,0.95)},
      anchor=north east,
      draw=none,
      fill=none,
      legend columns=2,
      cells={align=left},
      align=left,
      legend cell align={left}
    },
    grid=major, grid style={dotted, gray!40},
    thick,
    tick style={thick},
]

% PIDM line (dark teal)
\addplot[
    color={rgb:red,0;green,128;blue,128},
    line width=1.5pt,
    error bars/.cd,
    y dir=both,
    y explicit,
] coordinates {
    (0, 0.045) +- (0, 0.016)
    (1, 0.163) +- (0, 0.031)
    (2, 0.188) +- (0, 0.017)
    (3, 0.339) +- (0, 0.019)
};
\addlegendentry{PIDM}

% SCD (Ours) line (orange)
\addplot[
    color={rgb:red,255;green,127;blue,14},
    line width=1.5pt,
    error bars/.cd,
    y dir=both,
    y explicit,
] coordinates {
    (0, 0.0022) +- (0, 0.0001)
    (1, 0.0034) +- (0, 0.0002)
    (2, 0.0044) +- (0, 0.0004)
    (3, 0.0025) +- (0, 0.0002)
};
\addlegendentry{SCD (Ours)}

% Vanilla line (blue/purple)
\addplot[
    color={rgb:red,94;green,79;blue,162},
    line width=1.5pt,
    error bars/.cd,
    y dir=both,
    y explicit,
] coordinates {
    (0, 0.0041) +- (0, 0.0001)
    (1, 0.0045) +- (0, 0.0001)
    (2, 0.0058) +- (0, 0.0002)
    (3, 0.0034) +- (0, 0.0001)
};
\addlegendentry{Vanilla}

% DPS line (cyan)
\addplot[
    color={rgb:red,77;green,175;blue,174},
    line width=1.5pt,
    error bars/.cd,
    y dir=both,
    y explicit,
] coordinates {
    (0, 0.0039) +- (0, 0.0001)
    (1, 0.0056) +- (0, 0.0001)
    (2, 0.0341) +- (0, 0.0002)
    (3, 0.0669) +- (0, 0.0001)
};
\addlegendentry{DPS}

% Nest within Fig. 5
\node (anchorpoint) at (axis cs: 0.7,0.15) {}; % <-- change this coordinate
% End nesting within Fig. 5

\end{axis}
\end{tikzpicture}

\tikz[remember picture,overlay]\node[anchor=south east] at (anchorpoint) {
\begin{minipage}{.08\columnwidth} % <-- tweak size
\resizebox{\textwidth}{!}{%
\begin{tikzpicture}
    % Define the angle alpha (in degrees)
    \def\alphaDeg{45}
    % Define a theta angle that's less than alpha
    \def\thetaDeg{25}

    % Minor grid
    %\draw[help lines, gray!10, step=0.25] (-1.5,-1.5) grid (1.5,1.5);
    % Major grid
    %\draw[help lines, gray!40, step=1] (-1.5,-1.5) grid (1.5,1.5);
    \fill[fill=white, opacity=.7,draw=black!80] (-1.5,-1.5) rectangle (1.5,1.5);

    % Calculate points on the unit circle
    \coordinate (center) at (0,0);
    \coordinate (posAlpha) at ({\alphaDeg}:1);
    \coordinate (negAlpha) at ({-\alphaDeg}:1);
    \coordinate (right) at (1,0);
    
    % Calculate intersection of theta ray with the straight line
    \coordinate (thetaPoint) at ({cos(\thetaDeg)}, {sin(\thetaDeg)});
    
    % Draw the chopped circle (two arcs)
    % Arc from +alpha to -alpha (going the long way around)
    \draw[thick, black] (\alphaDeg:1) arc (\alphaDeg:{360-\alphaDeg}:1);
    
    % Draw the straight line connecting the gap
    \draw[thick, black] (posAlpha) -- (negAlpha);
    
    % Draw coordinate axes for reference
    \draw[gray, dashed] (-1.1,0) -- (1.1,0) node[right] {$x$};
    \draw[gray, dashed] (0,-1.1) -- (0,1.1) node[above] {$y$};
    
    % Mark the center
    % \fill (center) circle (1pt) node[below left] {$O$};
    
    % Mark the points at +alpha and -alpha
    \fill (posAlpha) circle (0pt) node[above right] {$+\alpha$};
    \fill (negAlpha) circle (0pt) node[below right] {$-\alpha$};
    
    % Mark the theta point on the straight line
    % \fill (thetaPoint) circle (0pt) node[right] {$\theta$};
    
    % Add angle markers
    \draw[thin] (center) -- (posAlpha);
    \draw[thin] (center) -- (negAlpha);
    \draw[thin] (center) -- (right);
    \draw[thin, gray!60] (center) -- (thetaPoint);
    
    % Draw angle arcs
    % \pic[draw, angle radius=0.3cm, angle eccentricity=1.3, "$\alpha$"] 
    %     {angle = right--center--posAlpha};
    % \pic[draw, angle radius=0.25cm, angle eccentricity=1.4, "$\alpha$"] 
    %     {angle = negAlpha--center--right};
    \pic[draw, angle radius=0.4cm, angle eccentricity=1.5, "$\phi$", gray!60] 
        {angle = right--center--thetaPoint};
    
    % Add labels for the regions
    % \node at (0.7, 0.7) {\small Circle};
    % \node at (0.9, 0.15) {\small \color{red} Line};
    % \node at (-0.8, 0) {\small Gap: $|\theta| < \alpha$};
    % \node at (0.4, -0.4) {\small \color{gray!60} $\theta$ on line};
    
\end{tikzpicture}}
\end{minipage}%
};
    \vspace*{-1.5em}
    \caption{\emph{Top left:} Illustration of ``Chop''. Mean $W_{1}$ distances with 2 standard deviations on varying degrees of misspecification on ``Chop''. Vanilla and SCD keep steady $W_{1}$ values, indicating they can learn the true data distribution. PIDM and DPS consistently increase with higher misspecification.}
    \label{fig:circlesWass}
    \vspace*{-0.5em}
\end{wrapfigure}

We first showcase and analyze our method on a set of toy examples in \cref{sec:toy}. In \cref{sec:darcy}, we evaluate our method on the Darcy flow PDE, a common benchmark in the diffusion PDE-constrained literature \citep{jacobsen_cocogen_2025, bastek_physics-informed_2025}. Finally, in \cref{sec:helmholtz} we compare our method with FunDPS \citep{yao2025guided} to generate solutions for the Helmholtz equation, highlighting that using two different guidance objectives can heavily bias the distribution using FunDPS. For evaluation, we use $W_{1}$ distance in \cref{sec:toy}, while for \cref{sec:darcy} and \cref{sec:helmholtz} we use NLL to evaluate distributional fidelity and a correctly specified residual to evaluate physical plausibility. Details on training and parameterization are available in \cref{app:training}. We also use a modified loss function based on the EDM loss \citep{karras_elucidating_2022}: we observed that the models had issues refining the fine-grained details at lower noise levels, and modifying the distribution from which the noise levels are sampled significantly improved the results across all models. The details of the modification are available in \cref{app:modifiedLoss}.
\subsection{Illustrative examples}
\label{sec:toy}
We explore properties of our method on new variants of the toy data set introduced by \citet{bastek_physics-informed_2025}, introducing new misspecification modalities. The data is composed of points sampled from the unit circle centered at the origin, and we train a diffusion model to learn to produce samples on the circle. Given a sample $(x,y) \in \mathbb{R}^{2}$, we define a residual function for this setting with the following equation:\looseness-1
\vspace{-0.5em}
\begin{talign}
\label{eqn:circleRes}
    \mathcal{R}_{\text{circle}}(x, y) = \left( \sqrt{x^2 + y^2} - 1 \right)^2,
\end{talign}
The architecture and training details are shown in \cref{appdx:circlesTraining}. We use $l_{c} = \exp(-\mathcal{R}_{\text{circle}}(x, y))$ as our constraint term for these experiments. We evaluate the performance of vanilla diffusion, regularized diffusion (PIDM), DPS guidance and our method on a few examples of misspecification using Wasserstein-1 distance to quantify distributional fidelity. Specifically, we use \cref{eqn:circleRes} with PIDM and our method on variations of the circle, namely a circle with a ``dent'' on top and a circle that is ``chopped'' after a particular $x$ coordinate. For ``Dent'', we use a circle with a polynomial interpolation at the top, producing the shape seen in \cref{fig:toySetup}. Details on the Dent interpolation are available in \cref{appdx:circlesTraining}. For ``Chop'', as illustrated in \cref{fig:circlesWass}, we define an angle threshold $-\alpha < \phi < \alpha$ for which all points on the circle with angle $\phi$ have their $x$ coordinate projected to $\cos{\alpha}$. 

The results of using Vanilla, DPS guidance, PIDM, and our method on the standard circle and the two described misspecifications are summarized in \cref{tab:circlesMethods}. On the unit circle, Vanilla, DPS and PIDM achieve similar values, with SCD achieving a clearly superior average $W_{1}$ distance. For ``Dent'', all models have a higher $W_{1}$ distance compared to the unit circle, though we note that SCD still performs the best. PIDM mostly performs worse than DPS and SCD, observed in \cref{fig:dentfig1}, due the model not placing any samples on the dent. The issue is most prominent with ``Chop'', where the $W_{1}$ distance with PIDM is almost an order of magnitude higher than vanilla and SCD: as visible in \cref{fig:chops}, this is because it only learns the samples on the arc along the circle. 

\begin{wraptable}{l}{.55\textwidth}
    \centering\scriptsize
    \caption{Measured (\textbf{best, lowest}) $W_{1}$ distances from the true data distribution. Means with two standard deviations across 100 estimates with 1000 samples from each method. All values multiplied by $10^{-3}$.}\label{tab:circlesMethods}
    \begin{tabular}{lccc}
        \toprule
        \multicolumn{1}{l}{\bf Method}  & \multicolumn{1}{c}{\bf Unit circle} & \multicolumn{1}{c}{\bf Dent} & \multicolumn{1}{c}{\bf Chop ($\alpha=\frac{\pi}{2}$)}
        \\ \midrule 
        Vanilla & $ 3.91 \pm 0.18 $ & $ 7.53 \pm 0.46 $ & $3.34 \pm 0.11$ \\
        DPS & $3.97\pm0.65$ & $6.27\pm0.63$ & $10.2\pm3.21$ \\
        PIDM & $4.04 \pm 0.20$ & $6.65 \pm 1.44 $ & $33.75 \pm 3.37$ \\
        SCD (ours) & $ \mathbf{2.16 \pm 0.15} $ & $\mathbf{5.6 \pm 0.44}$ & $\mathbf{2.42 \pm 0.16}$ \\
        \bottomrule
    \end{tabular}
\end{wraptable}

To measure the effect of varying misspecifications on the toy data, we vary the angle for ``Chop'' and note the change on the $W_{1}$ distance, pictured in \cref{fig:circlesWass}. Both vanilla diffusion and our method stay relatively steady for all different $\alpha$, while DPS and PIDM steadily increase with higher misspecification. Note the $W_1$ distance for PIDM is worse than vanilla diffusion even when the constraint is correctly specified at $\alpha=0$, due to the distribution being biased \emph{along the circle} even if the constraint is satisfied. Similar issues were visually noticed in \cite{bastek_physics-informed_2025} when using a high regularizer weight. 
\subsection{Darcy Flow}
\label{sec:darcy}
The task is to learn how to produce samples of a permeability field $K(x,y)$ and a pressure field $p(x,y)$ with $(x,y) \in \mathbb{R}^{2}$ that satisfy the following differential equation:
\begin{talign}
\label{eqn:darcy_residual}
    \mathcal{R}_{\text{Darcy}}(K, p) = \nabla \cdot \left( K \nabla p \right) + f = 0 ,
\end{talign}
where $f$ is the measurement of some fluid's flow through a porous medium, and corresponds to the divergence of the vector field defined by $K\nabla p$ pictured in \cref{tab:misspec-residuals}, or the \emph{net} amount of fluid entering and exiting a given point. Sources of misspecification when using this equation come from either assumptions about the material and potentially noisy measurements of the fluid's flow and applied pressure field. Details on Darcy Flow are given in \cref{app:DarcyFlowWriteup}.

This differential equation doubles as a residual function we use to verify our generated solutions. The differential terms are estimated through finite differences using the same stencils used by  \citet{bastek_physics-informed_2025}, meaning $K$ and $p$ are sampled as matrices in $\mathbb{R}^{n\times n}$. We use this to define our constraint adjustment $l_{c}=\exp(-\left| \mathcal{R}_{\text{Darcy}}(K, p) \right|)$, where $K$ and $p$ are the denoiser outputs. We use a diffusion model with a UNet backbone from \citet{karras_analyzing_2023}, and discretize $K$ and $p$ in $\mathbb{R}^{64 \times 64}$. The architecture, training details and runtimes are available in \cref{appdx:darcyTraining}. We highlight that the overhead in runtime between our method and a vanilla score matching implementation is low.

\begin{table*}[t]
  \centering\scriptsize
  \caption{ Left: Illustration of the Darcy Flow field. Right: Measured (\textbf{best/lowest}, \underline{second best}) mean absolute residuals and validation set NLL of the different methods across different levels of misspecification on Darcy. Residual means drawn across 1000 samples with two standard deviations, and NLL values (in bits/dim) on the complete validation set. Our method shows good performance across all misspecification levels. Vanilla does not use the constraint, so it is the same for all levels. }\vspace{-3pt}
  \label{tab:misspec-residuals}

\begin{minipage}[c]{.2\textwidth}
  \scriptsize
  \begin{tikzpicture}[inner sep=0]
  
    % Place the flow as background
    \node[anchor=south west,inner sep=0] (image) at (0,0)
      {\includegraphics[width=\linewidth,trim=5 5 5 5,clip]{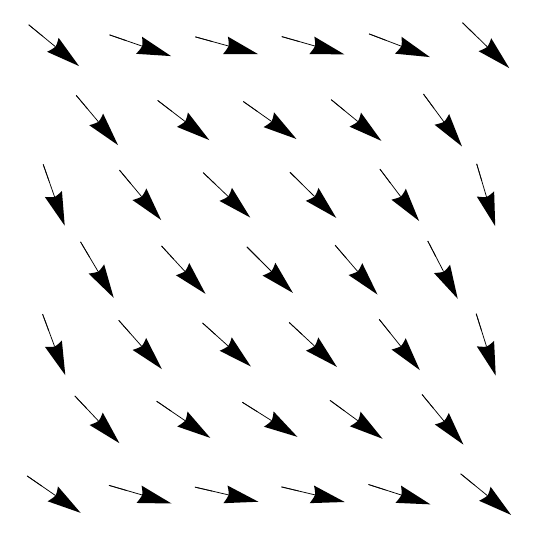}};

    % Set coordinate system to the image size
    \begin{scope}[x={(image.south east)},y={(image.north west)}]
      % Now (0,0) is bottom-left, (1,1) is top-right of the image

      % Domain boundaries
      \draw[thick,densely dashed] (0,0) -- (1,0) -- (1,1) -- (0,1) -- cycle;

      % Source/sink
      \fill[draw,fill=red,rounded corners=1pt] (0,1) -- ++(.1,0) -- ++(0,-.1) -- ++(-.1,0) -- cycle;
      \fill[draw,fill=blue,rounded corners=1pt] (1,0) -- ++(-.1,0) -- ++(0,.1) -- ++(.1,0) -- cycle;
      
      % Labels
      \node[fill=white,fill opacity=.8,rounded corners=2pt,draw=black!50,inner sep=3pt] at (.5,.5) {$-K\nabla p$};
      \node[anchor=west,inner sep=1pt,outer sep=3pt,font=\bf,black,fill=white] at (.1,.95) {Source};
      \node[anchor=east,inner sep=1pt,outer sep=3pt,font=\bf,black,fill=white] at (0.9,.05) {Sink};      
      
    \end{scope}
  \end{tikzpicture} 
\end{minipage}
\hfill
\begin{minipage}[c]{.78\textwidth}
\raggedleft\scriptsize
\setlength{\tabcolsep}{1pt}
\begin{tabular}{llcccc}
\toprule
& & \tikz[baseline=0pt]\draw[anchor=base,black,<-] (0,0) --node[yshift=2pt,font=\scriptsize\em]{increasing misspecification} ++(2cm,0); & \em Original & \multicolumn{2}{l}{\tikz[baseline=0pt]\draw[anchor=base,black,->] (0,0) --node[yshift=2pt,font=\scriptsize\em]{increasing misspecification} ++(3.5cm,0);} \\
\bf Metric &
\bf Method & 
\bf $f = 7.5$ & 
\bf $f = 10$ &
\bf $f = 12.5$ & 
\bf $f = 15$ \\ 
\midrule
& Vanilla & $0.157 \pm 0.071$ & $0.157 \pm 0.071$ & $0.157 \pm 0.071$ & $0.157 \pm 0.071$ \\
& DPS & $\underline{ 0.141 \pm 0.063 }$ & $0.139 \pm 0.067$ & $0.140 \pm 0.065$ & $\underline{0.139 \pm 0.066}$ \\
Residual & PIDM & $0.184 \pm 0.008 $ & $\mathbf{0.025 \pm 0.010} $ & $\underline{0.129 \pm 0.008} $ & $0.205 \pm  0.009 $ \\
& SCD (ours) & $\mathbf{0.113 \pm 0.048}$ & $\underline{0.106 \pm 0.049}$ & $\mathbf{0.111 \pm 0.059}$ & $\mathbf{0.105 \pm 0.048}$ \\

\midrule
& Vanilla & $\mathbf{-15.5}$ & $\mathbf{-15.5}$ & $\mathbf{-15.5}$ & $\mathbf{-15.5}$ \\
& DPS & $ -14.5 $ & $ -14.5 $ & $ -14.5 $ & $ -14.3 $ \\
NLL & PIDM & $ -3.9 $ & $ -3.7 $ & $ -3.9 $ & $ -3.5 $ \\
& SCD (ours) & $ \underline{-15.1} $ & $ \underline{-15.2} $ & $ \underline{-15.1} $ & $ \underline{-15.1} $ \\
\bottomrule
\end{tabular}
\end{minipage}
\vspace{-3pt}
\end{table*}

\paragraph{Distributional fidelity, misspecification and residuals} Darcy Flow is a mathematical model particularly used to infer properties of a material in a real physical system. As such, using this model can be prone to different sources of error. In this section, we present experiments on a simple case where the measured flow is on a wrong scale (an example of measurement miscalibration).

To induce misspecification, we test different values for the measured flow $f$ in the constraint, while keeping the original data. The results using vanilla, DPS guidance, PIDM and our method are shown in \cref{tab:misspec-residuals} reported with residual values and NLL estimated through the probability flow ODE, as specified in \cref{app:NLLcalc}, with qualitative samples shown in \cref{fig:darcy_samplesresiduals}. The guidance scale used for DPS was tuned to $0.02$ based on the residual values with grid search. We highlight that PIDM is especially sensitive to the induced misspecifications, causing its residuals to go up considerably to the point that it performs worse than a baseline vanilla diffusion model. On the other hand, our model mostly preserves the NLL. We hypothesize that it can still use the information in the area where the $f$-field is correctly specified as zero, while learning to adapt the gradient information used for the source and the sink. Using a guidance method with vanilla diffusion slightly improves the residuals but not by a significant margin, and the same behavior continues across different misspecification levels.

\subsection{Helmholtz Equation}
\label{sec:helmholtz}

\begin{wrapfigure}{r}{.4\textwidth}
    \centering
    \includegraphics[width=\linewidth]{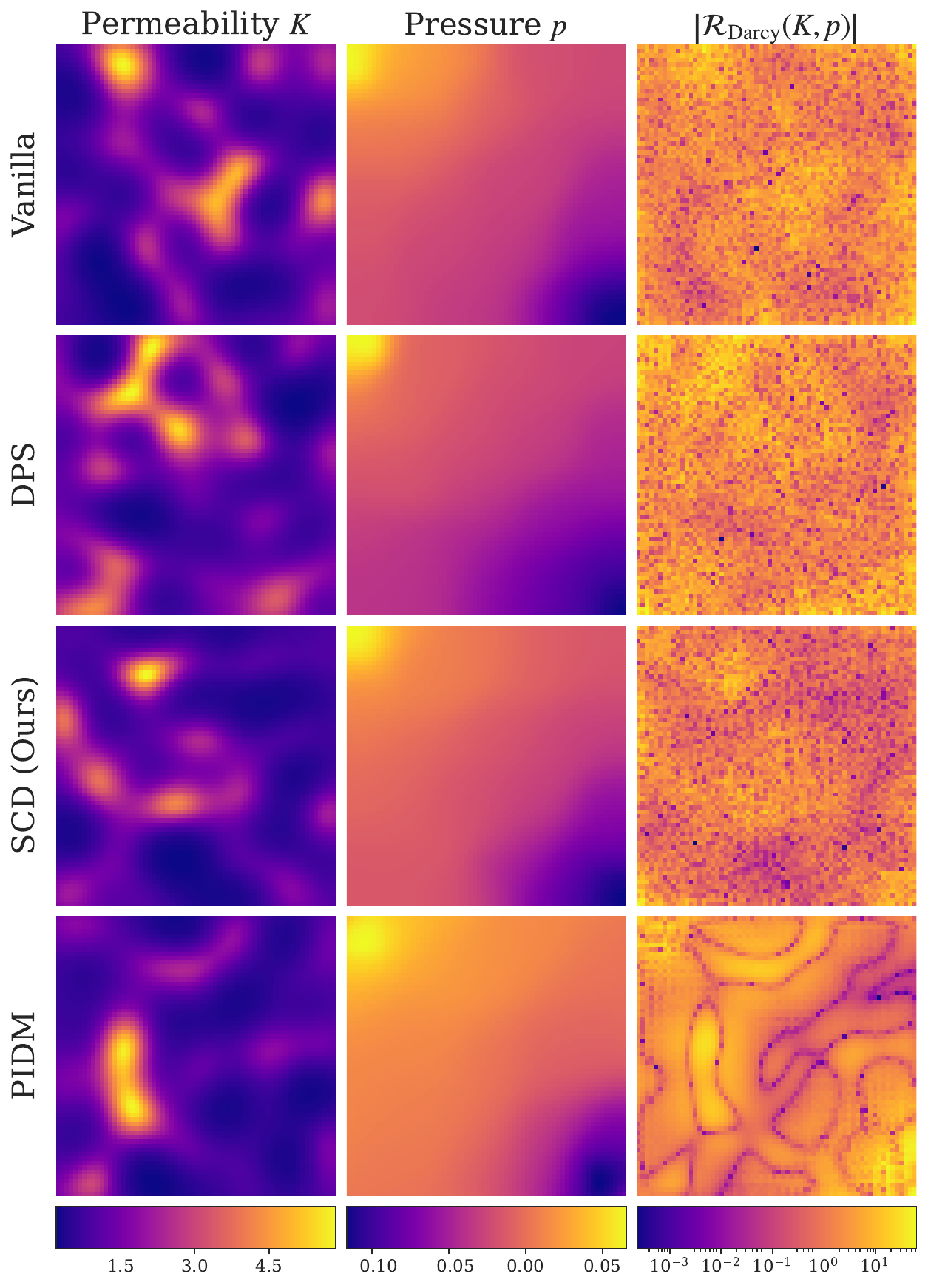}
    \caption{Residuals and samples produced by vanilla, DPS, SCD and PIDM. The average residuals with PIDM are lower, but regions with low residual magnitude are more localized. SCD has an overall more uniform lower residual.}
    \label{fig:darcy_samplesresiduals}
    \vspace{-0.5em} % <-- do not squeeze this much
\end{wrapfigure}

The Helmholtz Differential Equation is used to model the propagation of waves through (possibly heterogeneous) media. Its 2-dimensional version is given by:
\begin{talign}
\label{eqn:HelmholtzEqn}
    \mathcal{R}_{\text{Helmholtz}}(u, a) = \nabla^{2}u(\bm{x}) + k^2u(\bm{x}) - a(\bm{x}) = 0, 
\end{talign}
where $\bm{x} \in (0,1)^2$ and $u$ is the solution field, $a$ the source field and $k$ the wave number. The solution is given by a spatial wave in 2 dimensions and boundary conditions equal to 0. This equation is particularly important in acoustics and electromagnetics, describing the way sound and light travel through space. It is generally subject to the measurement of $k$, which uniquely defines the solution of the system given the boundary conditions. The wave number is usually estimated with sensor readings in noisy media, which subjects the equation to observation error.

Similar to Darcy Flow, we use \cref{eqn:HelmholtzEqn} to define $l_{c} = \exp(-\left| \mathcal{R}_{\text{Helmholtz}}(u, a) \right|)$, where $u$ and $a$ correspond to the denoiser outputs. Following \citet{yao2025guided}, we use a discretization of $u$ and $a$ in $\mathbb{R}^{128\times 128}$ and a wave number $k=1$. The training details are available in \cref{appdx:helmholtzTraining}.

\begin{table*}[t]
  \centering\scriptsize
  \caption{Measured (\textbf{best/lowest}, \underline{second best}) mean absolute residuals and validation set NLL of the different methods across different levels of misspecification on Helmholtz. Residual means drawn across 1000 samples with two standard deviations, and NLL values (in bits/dim) on the complete validation set. Our method shows good performance across all misspecification levels. Vanilla does not use the constraint, so it is the same for all levels.} \vspace{-3pt}
  \label{tab:misspec-residuals-helmholtz}
  \setlength{\tabcolsep}{10pt}
\begin{tabular}{llcccc}
\toprule
& & \em Original & \multicolumn{3}{l}{\tikz[baseline=0pt]\draw[anchor=base,black,->] (0,0) --node[yshift=2pt,font=\scriptsize\em]{increasing misspecification} ++(7cm,0);} \\
\bf Metric &
\bf Method & 
\bf $\sigma_\mathrm{obs} = 0 $ & 
\bf $\sigma_\mathrm{obs} = 0.05$ &
\bf $\sigma_\mathrm{obs} = 0.1$ & 
\bf $\sigma_\mathrm{obs} = 0.5$ \\ 
\midrule
& Vanilla & $0.035 \pm 0.007$ & $0.035 \pm 0.007$ & $0.035 \pm 0.007$ & $0.035 \pm 0.007$ \\
Residual & FunDPS & $13.84 \pm 3.4$ & $ 13.86 \pm 3.4$ & $13.84 \pm 3.3$ & $14.2 \pm 3.4$ \\
& DPS & $ \underline{0.032 \pm 0.004} $ & $ \underline{0.033 \pm 0.005} $ & $ \underline{0.035 \pm 0.004} $ & $\underline{0.034 \pm  0.003} $ \\
& SCD (ours) & $\mathbf{0.025 \pm 0.003}$ & $ \mathbf{0.024 \pm 0.003} $ & $ \mathbf{0.024 \pm 0.002} $ & $\mathbf{0.023 \pm 0.003}$ \\

\midrule
& Vanilla & $\mathbf{-21.6}$ & $\mathbf{-21.6}$ & $\mathbf{-21.6}$ & $\mathbf{-21.6}$ \\
NLL & FunDPS & $-7.2$ & $-7.3$ & $-7.2$ & $-7.2$ \\
& DPS & $\mathbf{-21.6}$ & $\underline{-21.5}$ & $\underline{-21.5}$ & $\underline{-21.5}$ \\
& SCD (ours) & $\mathbf{-21.6}$ & $\underline{-21.5}$ & $\underline{-21.5}$ & $\underline{-21.5}$ \\
\bottomrule
\end{tabular}
\vspace{-3pt}
\end{table*}

\paragraph{Distributional fidelity, misspecification and residuals}
In this setting, we perturb the wave number adding Gaussian noise. The degree of misspecification is handled by the standard deviation of this noise, and this experiment allows us to measure how sensitive SCD can be to noisy gradients through the constraint function. The results are summarized in \cref{tab:misspec-residuals-helmholtz}, where we compare our results with the guidance framework FunDPS using a model trained by the authors \citep{yao2025guided}, a vanilla diffusion model and DPS guidance on the vanilla diffusion model. The FunDPS sampling is done with a combination of a reconstruction loss based on partial observations and the residual to observe the issues resulting from imbalanced guidance. SCD manages to preserve the NLL with the most competitive residual values in this task. It is evident that using this combined guidance in FunDPS heavily favors the reconstruction error, essentially forfeiting the physical consistency of the generated samples. On the other hand, regular DPS only sees marginal improvements in terms of the residuals compared to the vanilla diffusion model.
\section{Discussion and conclusion}
\label{section:discussion}
In this work, we introduced Softly Constrained Denoisers (SCD), a simple way to embed constraint knowledge directly into diffusion model denoisers. Unlike guidance-based methods or regularization-based methods that bias the training distribution, our approach provides a soft inductive bias that improves constraint satisfaction while retaining flexibility when constraints are misspecified. We demonstrated these benefits through illustrative examples (\cref{sec:toy}) and PDE benchmarks (\cref{sec:darcy,sec:helmholtz}), showing that SCD outperforms standard diffusion models and existing constraint-enforcing approaches when the constraint is partially misspecified. Our results suggest that SCD can serve as a drop-in upgrade for diffusion-based generative modeling in applications where constraints are desired.\looseness-1

\paragraph{Limitations and future work}
A limitation of our method is that our denoiser is based on an approximation to the score function for the tilted distribution $l_{c}(\bm{x}_{0})p(\bm{x}_{0})$, which requires a differentiable constraint $l_{c}(\bm{x}_{0})$. There are a multitude of applications that have such constraints available, but to further extend the usefulness of our approach, future work may explore how to derive differentiable constraints from non-differentiable ones. We also note that, while standard in ``diffusion for PDEs'' literature, we assume access to clean data. Finally, even though the inductive bias provided by the denoiser parameterization does not inhibit the diffusion model from following the data distribution, the parameterization may push the model towards particular biases in practice, for example CNNs with their inductive bias toward translation equivariance.

Future work could explore deriving SCDs based on more advanced guidance approximations and applying them in various scientific application areas. We also highlight that the trade-off between distributional fidelity and constraint compliance produces a Pareto front; our work explores a point where we desire to maximize distributional fidelity while using the constraint, and future work may focus on how to tune ``where'' in the front we land, e.g. we want stronger (but not strict) constraint compliance, at a small expense of data fidelity. Further robustness to misspecification could be achieved by introducing learnable parameters to $l_c(\bm x)$ itself. We also believe that a study on the effects of noise in the training data could be highly valuable for practitioners and the ``diffusion for PDEs'' field as a whole, as the clean data assumption may limit applicability in practice.

\paragraph{Impact statement} Diffusion models have garnered attention to generate samples from the solution space in Inverse Problems \citep{chung_diffusion_2023, song_pseudoinverse-guided_2023} and Differential Equations \citep{jacobsen_cocogen_2025, bastek_physics-informed_2025, huang2024diffusionpde}, where this work is framed to make the most impact. The main contribution is a denoiser architecture that seeks to balance data and prior constraint information to improve sample diversity and reduce sensitivity to constraint misspecification, which are paramount for scientific applications.

The authors foresee this work to be useful in physical science applications such as fluid dynamics and climate forecasting. However, given the nature of research around differential equations and dynamical systems, there may also be use for military applications, though the authors do not see a direct relation at the time of presenting the work.

\section*{Acknowledgements}

This work was supported by the Research Council of Finland (Flagship programme: Finnish Center for Artificial Intelligence FCAI and decision 359207), ELLIS Finland, EU Horizon 2020 (European Network of AI Excellence Centres ELISE, grant agreement 951847). VYS acknowledges funding from the Finland Fellowship awarded by Aalto University. SK acknowledges funding from the UKRI Turing AI World-Leading Researcher Fellowship (EP/W002973/1). AS and SR acknowledge funding from the Research Council of Finland (grants: 339730 and 362408). MS acknowledges funding from AI Hub in Generative Models by Engineering \& Physical Sciences Research Council (EPSRC) with Funding Body Ref: EP/Y028805/1.

The authors acknowledge the research environment provided by ELLIS Institute Finland, the CSC -- IT Center for Science, Finland, and the Aalto Science-IT project for computational resources.

\bibliographystyle{icml2026}

\bibliography{main}

%%%%%%%%%%%%%%%%%%%%%%%%%%%%%%%%%%%%%%%%%%%%%%%%%%%%%%%%%%%%%%%%%%%%%%%%%%%%%%%
%%%%%%%%%%%%%%%%%%%%%%%%%%%%%%%%%%%%%%%%%%%%%%%%%%%%%%%%%%%%%%%%%%%%%%%%%%%%%%%
% APPENDIX
%%%%%%%%%%%%%%%%%%%%%%%%%%%%%%%%%%%%%%%%%%%%%%%%%%%%%%%%%%%%%%%%%%%%%%%%%%%%%%%
%%%%%%%%%%%%%%%%%%%%%%%%%%%%%%%%%%%%%%%%%%%%%%%%%%%%%%%%%%%%%%%%%%%%%%%%%%%%%%%
\newpage
\clearpage

\appendix
\onecolumn

\section*{Appendices}

\section{Proofs} \label{app:proofs}

\PropOptimalDenoiserSCD*

\begin{proof}
    Take the denoising score matching loss, letting $D^{\text{orig}}_{\theta}(x_{t},t)=y$, and evaluate it with SCD:
    \begin{equation}
        J(y) = \mathbb{E}_{\vx_{0}|\vx_{t}} \left[ \left\lVert y + \gamma_{\theta}(\vx_{t},t)\sigma(t)^{2}\nabla_{y}\log l_{c}(y) - \vx_{0} \right\rVert^2 \right],
    \end{equation}
    where we will denote $\gamma_{\theta}(\vx_{t}, t)\sigma(t)^2 = \eta(t)$ and $\nabla_{y}\log l_c(y)=v(y)$ for brevity (we will unroll this once we get to the final expression). We can expand the norm as follows:
    \begin{equation}
    \label{eq:proofExpandNorm}
        \mathbb{E}_{\vx_{0}|\vx_{t}} \left[ \lVert y \rVert^2 + \lVert \eta(t)v(y)\rVert^2 + \lVert \vx_{0}\rVert^2+2\eta(t)v(y)^{\intercal}y-2\vx_0^{\intercal}y-2\eta(t)\vx_{0}^{\intercal}v(y) \right].
    \end{equation}
    We first take the expectation and then the gradient with respect to $\eta(t)$ and set it equal to zero to study its stationary points:
    \begin{talign}
        \nabla_{\eta(t)}J(y)&=2\left( \eta(t)v(y)^{\intercal}v(y)+v(y)^{\intercal}y-v(y)^{\intercal}\hat{\vx}^{*}_{0} \right)=0\\
        &\rightarrow v(y)^{\intercal}(y+\eta(t)v(y)-\hat{\vx}^{*}_{0})=0,
    \end{talign}
    where we define $\hat{\vx}^{*}\sb{0} = \mathbb{E}[\vx\sb{0}|\vx\sb{t}] $ as the ideal denoiser. A solution may come from $v(y)=0$, which is the case where the constraint is perfectly specified with an ideal denoiser, or a denoiser fit to a misspecified constraint. We are interested in the case given by the second factor. Unrolling definitions:
    \begin{equation}
    \label{eq:stationarityGamma}
        D^{\text{SCD}}_\theta(\bm x_t, t)\vcentcolon=y+\gamma_{\theta}(\vx_{t},t)\sigma(t)^{2}\nabla_{y}\log l_ c(y)=\hat{\vx}_{0}^{*},
    \end{equation}
    giving the stationarity conditions with respect to $\gamma_{\theta}$.
    
    Now we take \cref{eq:proofExpandNorm} and take the gradient with respect to $y$, yielding:
    \begin{equation}
        \begin{split}
            \nabla\sb y J(y) = 2(y&+\eta(t)^{2}(\nabla\sb{y}v(y))^{\intercal}v(y)+\eta(t)\left(\nabla\sb{y}v(y)^{\intercal}y+v(y)\right)\\
            &-\mathbb{E}[\vx\sb{0}|\vx\sb{t}]-\eta(t)\nabla\sb{y}v(y)^{\intercal} \mathbb{E}[\vx\sb{0}|\vx\sb{t}] ),
        \end{split}
    \end{equation}
    and note that $\nabla\sb{y}v(y) = H$ is a Hessian with respect to $y$ for $\log l\sb{c}(y)$. Grouping terms by $y$, $v(y)$ and $\hat{\vx}^{*}\sb{0}$, we get:
    \begin{talign}
        (I+\eta(t)H)y&+(I+\eta(t)H)\eta(t)v(y)-(I+\eta(t)H)\hat{\vx}^{*}\sb{0}\\
        &=(I+\eta(t)H)\left(y+\eta(t)v(y)-\hat{\vx}^{*}\sb{0}\right)=0.
    \end{talign}
    Unless $(I+\eta(t)H)$ is singular, the stationarity solution can only be given by the second factor. This gives:
    \begin{talign}
        y+\eta(t)v(y)-\hat{\vx}^{*}\sb{0}&=y+\gamma\sb{\theta}(\vx\sb{t},t)\sigma(t)^{2}\nabla\sb{y}\log l\sb{c}(y)-\hat{\vx}^{*}\sb{0}=0\\
        \rightarrow D^{\text{SCD}}_\theta(\bm \vx_t, t) &\vcentcolon=y+\gamma\sb{\theta}(\vx\sb{t},t)\sigma(t)^{2}\nabla\sb{y}\log l\sb{c}(y)=\hat{\vx}^{*}\sb{0},
    \end{talign}
    concluding the proof.
\end{proof}

\CorolGammaRole*

\begin{proof}
    Take the stationarity condition for $\gamma_{\theta}$ given by \cref{prop:scd_optimal_denoiser}, with $D^{\text{orig}}_{\theta}(\vx_{t},t)=y$ and the (generally unknown) ideal denoiser $\hat{\vx}^{*}\sb{0}$:
    \begin{equation}
        v(y)^{\intercal}(y+\eta(t)v(y)-\hat{\vx}^{*}_{0})=0
    \end{equation}
    Distributing the multiplication and rearranging terms, we get:
    \begin{talign}
        \eta(t)v(y)^{\intercal}v(y) &= v(y)^{\intercal}(\hat{\vx}^{*}_{0}-y)\\
        \rightarrow \eta(t)\lVert v(y)\rVert^{2} &= \langle v(y), \hat{\vx}^{*}\sb{0}-y\rangle\\
        \rightarrow \eta(t) &= \frac{\langle v(y), \hat{\vx}^{*}\sb{0}-y\rangle}{\lVert v(y)\rVert^{2}},
    \end{talign}
    where we use the definition of inner product from the first to the second line. Multiplying both sides by $v(y)$, and recalling the definitions $\gamma_{\theta}(\vx_{t}, t)\sigma(t)^2$ and $v(y)=\nabla_{y}\log l_c(y)$ yields the result:
    \begin{talign}
        \gamma\sb{\theta}(\vx\sb{t},t)\sigma(t)^{2}\nabla\sb{y}\log l\sb{c}(y) = \frac{\langle \nabla\sb{y}\log l\sb{c}(y), \hat{\vx}^{*}\sb{0}-y\rangle}{\lVert \nabla\sb{y}\log l\sb{c}(y) \rVert^{2}}\nabla\sb{y}\log l\sb{c}(y) = \text{proj}_{\nabla\sb{y}\log l\sb{c}(y)}(\hat{\vx}^{*}\sb{0}-y).
    \end{talign}
    In particular, $\sigma(t)^{2}$ acts as a preconditioner for the norm of $(\hat{\vx}^{*}_{0}-y)$, and thus $\gamma_{\theta}$ encodes the proportion of the denoiser error $\hat{\vx}^{*}\sb{0}-y$ that can be explained by the constraint gradient $\nabla\sb{y}\log l\sb{c}(y)$. It then follows that its value goes to zero in two cases: $y=\hat{\vx}^{*}\sb{0}$, meaning that the denoiser is already ideal, or when the denoiser error is perpendicular to the constraint gradient, meaning that the constraint does not explain any part of the error. In a sense, this parameter serves a similar role in the optimization to the learned uncertainty weight in the EDM-2 loss \citep{karras_analyzing_2023}.
\end{proof}

% This prints the proposition again with the EXACT same number as before
\PropShiftedOptimum*

\begin{proof}
    Let $y = D_\theta(x_t, t)$. The objective can be rewritten as the expected risk for a given input $x_t$:
    \begin{equation}
        J(y) = \mathbb{E}_{x_0|x_t} \left[ \| y - x_0 \|^2 \right] + \lambda \| \mathcal{R}(y) \|^2.
    \end{equation}
    To find the optimum $y^*$, we take the gradient with respect to $y$ and set it to zero:
    \begin{align}
        \nabla_y J(y) &= \nabla_y \left( \|y\|^2 - 2y^\top \mathbb{E}[x_0|x_t] + \mathbb{E}[\|x_0\|^2] \right) + \lambda \nabla_y \| \mathcal{R}(y) \|^2 \\
        0 &= 2(y - \mathbb{E}[x_0|x_t]) + 2\lambda (\nabla_y \mathcal{R}(y))^\top \mathcal{R}(y).
    \end{align}
    Solving for $y$ yields the result. Consequently, $D^*_{\text{reg}}(x_t, t) \neq \mathbb{E}[x_0|x_t]$ unless $\lambda=0$ or the constraint gradient is zero.
\end{proof}

\PropIncreasedELBO*

\begin{proof}

    The standard diffusion loss $\mathcal{L}_{\text{diff}} = \mathbb{E}[\| D_\theta(x_t, t) - x_0 \|^2]$ corresponds to the variational bound on the negative log-likelihood if using a specific weighting $w(t)$\citep{kingma2021variational}. It is a well-known result that the unique global minimizer of this quadratic loss is the conditional expectation $D^*_{\text{vanilla}}(x_t, t) = \mathbb{E}[x_0 | x_t]$.
    
    Since $D^*_{\text{reg}}(x_t, t) \neq \mathbb{E}[x_0 | x_t]$ (from Proposition \ref{prop:shifted_optimum}), and $\mathcal{L}_{\text{diff}}$ is strictly convex with respect to the prediction $y$, it follows that:
    \begin{equation}
    \mathbb{E}_{x_0|x_t} \left[ \| D^*_{\text{reg}}(x_t) - x_0 \|^2 \right] > \mathbb{E}_{x_0|x_t} \left[ \| D^*_{\text{vanilla}}(x_t) - x_0 \|^2 \right].
    \end{equation}
    Specifically, by the bias-variance decomposition, the increase in the diffusion loss is exactly the squared magnitude of the shift derived in Proposition \ref{prop:shifted_optimum}:
    \begin{equation}
        \Delta \mathcal{L} = \| D^*_{\text{reg}}(x_t, t) - \mathbb{E}[x_0|x_t] \|^2.
    \end{equation}
\end{proof}

\section{Extended related work}
\label{app:extRelatedWork}

\paragraph{Diffusion models applied to PDEs } 
\citet{jacobsen_cocogen_2025} propose conditional PDE generation using a Controlnet-like conditioning structure \citep{zhang2023adding} and an inference-time adjustment where the final samples are optimized to have a small PDE residual. 
\citet{bastek_physics-informed_2025} present Physics-Informed Diffusion Models (PIDM), a framework to train DDPM-based diffusion models with a PDE residual as a regularizer term to minimize along the loss function. 
Several works utilize DPS-like guidance \citep{chung_diffusion_2023} for PDE data assimilation, targeting noisy measurements \citep{shysheya2024conditional}, constraint satisfaction \citep{huang2024diffusionpde}, or infinite-dimensional Banach spaces \citep{yao2025guided}. Similarly, \citet{chenggradient} employ projection-based sampling \citep{zhu_plugnplayrestore_2023}. Unlike these, our method avoids approximate inference-time guidance. Furthermore, our parameterization is orthogonal to recent architecture-focused works on neural operators \citep{huwavelet, oommen2024integrating} or GNNs \citep{valencia2025learning}, as it remains compatible with any base architecture.

\paragraph{Injecting measurement structure for training inverse problem solvers} Mathematically, the closest work is the likelihood-informed Doob's h-transform by \citet{denker2024deft}, who finetune adapters using observation gradients $\nabla_{\bm x_0} p(\bm y\mid \bm x_0)$, similar to our constraint-informed parameterization using $\nabla_{\bm x_0} l_c(\bm x_0)$. However, our motivation and methodology differ: they use likelihoods $p(\bm y\mid \bm x_0)$ for Bayesian inference from noisy observations, whereas we \emph{define} $l_c(\bm x_0)$ to restrict generation to a constrained subset. Their goal is an alternative to inference-time adjustment, while we seek to alleviate distributional biases and constraint misspecification. Furthermore, we train from scratch, whereas they finetune adapters on larger models. Finally, \citet{liu20232, chung2023direct} propose embedding structure via bridge processes that interpolate between condition and target; this is inapplicable to our setting, as we do not target a translation problem.

\paragraph{Inference time adjustment for inverse problems} Many methods target adjusting diffusion models at inference time to solve inverse problems, many of them formally targeting approximation of \cref{eq:guidance_adjustment}. \citet{song_score-based_2021} was one of earliest papers to propose inference-time adjustments to the diffusion model to solve problems like inpainting and color restoration. \citet{chung2022improving, wangzero, zhu_plugnplayrestore_2023} propose more advanced methods for a wider range of inverse problems. The first methods to make the explicit connection to \cref{eq:guidance_adjustment} were \citet{ho_video_2022, chung_diffusion_2023, song_pseudoinverse-guided_2023}. While the method by \citet{song_pseudoinverse-guided_2023} only worked for linear inverse problems, \citet{song_loss-guided_2023} generalized it to general guidance functions through Monte Carlo integration of \cref{eq:guidance_adjustment}. Works focused on improving the $p(\bm x_0\mid \bm x_t)$ approximation include \citep{boys_tweedie_2024, peng_improving_2024, rissanen_free_2025}. 

\paragraph{Hard-constraint diffusion via modified dynamics} Several works impose constraints by \emph{changing the diffusion dynamics} so that the support is restricted by construction. Riemannian diffusion models and flow matching models move the noising and denoising processes to a target manifold, enabling sampling on spheres, tori, hyperboloids, and matrix groups but requiring smooth geometry and geometric operators \citep{de2022riemannian,huang2022riemannian,chenflow}. \citet{fishmandiffusion, fishman2023metropolis, lou_reflected_2023} propose to use noising processes that are constrained within convex sets defined by inequality constraints, while \citet{liu_mirror_2023} tackle the problem by learning standard diffusion models in a dual space created using a mirror map. \emph{Star-shaped DDPMs} tailor noise to exponential-family distributions suited to constrained manifolds \citep{okhotin2023star}. These methods provide hard constraints but are typically specialized to particular geometries or constraint classes. Other works have explored modifying the sampling dynamics in other non-geometric constrained settings \citep{Liu2023LearningDB,Christopher2024ConstrainedSW,saeid2025ManuallyBridged}. These prior works still introduce the constraint information at sampling time, meaning the constraint information is not given at training time, and focus on hard constraint guarantees. Thus, they still are prone to the misspecification issues that we tackle. Contemporary work by \citet{liang2026improvedconstrainedgenerationbridging} approaches the sampling dynamics with a similar architecture structure as ours, but still focuses on the inference-time adaptation setting. Other contemporary work by \citet{christopher2026constraintawareflowmatchingdecision} explores a similar approach to ours in the flow matching framework, where they introduce fixed point iteration differentiable layers to the denoiser network, which are derived from the constraint through Sequential Quadratic Programming. Their setting is different from ours however, as they do not address the problem of misspecification. \looseness-2

\paragraph{Optimizing samples to match with constraints} \citet{benhamu_dflow} generate samples with constraints by optimizing a source point in the noisy latent space such that the generative ODE solution matches with the constraint. \citet{tang2024tuning} instead optimize the noise injected during the stochastic sampling process. \citet{pooledreamfusion} generate samples by directly optimizing the target image within a constrained space (e.g., images parameterized by a neural radiance field \citep{mildenhall2021nerf}), while minimizing the diffusion loss for the image. \citet{wang2023prolificdreamer} extend the method with a particle-based variational framework. In a similar manner, \citet{mardanivariational} formulate the sampling process as optimizing a variational inference distribution on the clean samples. \looseness-2

\paragraph{Distributional constraints} \citet{khalafi2024constrained} formulate a distributional constrained generation task where the generative distribution should have a KL divergence below a threshold to a set of auxiliary distributions $q^i$. \citet{khalaficomposition} extend the idea to compositional generation.

\paragraph{Soft inductive biases}
\citet{finzi_soft_2021} propose using ``dual path'' layers to build a neural network, where one path uses a hard-constrained layer, e.g.\ a rotationally equivariant layer, and the other uses a more ``relaxed'' layer. By assigning a lower prior probability to the relaxed path, a soft inductive bias towards solutions that satisfy the constraint is imposed, without restricting the possible hypothesis space for the neural network.

\section{Negative log-likelihood estimation}
\label{app:NLLcalc}

For a diffusion model with denoiser $D_{\vtheta} (\vx_{t}, t)$, the log-likelihood can be estimated from the probability flow differential equation and using an instantaneous change of variables \cite{song_likelihood_2021,kong_infotheo_2023}, from which we get the following formula:
\begin{equation}
\label{eqn:loglikelihood}
    \log p(\vx_{0}) = \log p(x_{T}) + \int_{0}^{T}\nabla_{\vx_{t}}\cdot g(\vx_{t}, t) \text{d}t,
\end{equation}
where $g(t) = \dot{\sigma}(t)\sigma(t)\nabla_{\vx_{t}} \log p(\vx_{t})$ is the drift of the probability flow ODE. We follow \citet{karras_elucidating_2022} and set $\sigma(t)=t$, meaning that the drift is defined by the score function, which is approximated as $\nabla_{\vx_{t}} \log p(\vx_{t}) \approx \frac{D_{\vtheta} (\vx_{t}, t) - \vx_{t}}{\sigma_{t}^{2}}$ for score matching diffusion models \citep{song_score-based_2021, karras_elucidating_2022}, and $T=\sigma_{\text{max}}^{2}$ the maximum noise level of the diffusion process. We take $p(\vx_{T})$ to be approximately $\mathcal{N}(0, \sigma_{\text{max}}^{2}\mathbf{I})$, i.e. the initial noise distribution, giving an analytic expression for $\log p(\vx_{T})$ as the entropy of a multivariate normal distribution. Then, for the divergence $\nabla_{\vx_{t}}\cdot g(t)$ we have:
\begin{align}
    \nabla_{\vx_{t}} \cdot \left( \nabla_{\vx_{t}} \log p(\vx_{t}) \right) &\approx \nabla_{\vx_{t}} \cdot \left( \frac{D_{\vtheta} (\vx_{t}, t) - \vx_{t}}{\sigma_{t}^{2}} \right)\\
    &= \frac{-d + \nabla_{\vx_{t}}\cdot D_{\vtheta} (\vx_{t}, t)}{\sigma_{t}^{2}},
\end{align}
where $d$ is the dimensionality of the data. Notice that $\nabla_{\vx_{t}}\cdot D_{\vtheta} (\vx_{t}, t)$ is the trace of the Jacobian of the denoiser. Computing the Jacobian exactly is expensive, requiring $O(d)$ network evaluations. To reduce the computational cost, we use Hutchinson's trace estimator:
\begin{align}
    \nabla_{\vx_{t}} \cdot D_{\vtheta}(\vx_{t}, t) &= \text{Tr}(\mathbf{J}_{D_{\vtheta}})\\
    &= \mathbb{E}_{\epsilon} \left[ \epsilon^{\intercal}\mathbf{J}_{D_{\vtheta}}\epsilon \right]\\
    &= \mathbb{E}_{\epsilon} \left[ \epsilon^{\intercal} \frac{\partial  D_{\vtheta} (\vx_{t}, t)}{\partial \vx_{t}} \epsilon \right],
\end{align}
where $\epsilon$ is a random vector with distribution such that $\mathbb{E}[\epsilon^{\intercal} \epsilon] = \mathbf{I}$, usually chosen to be normal or Rademacher. This estimator only requires $O(n_{\text{samples}})$ evaluations of the network forward pass using autodifferentiation, with $n_{\text{samples}}$ the number of $\epsilon$ samples drawn.

In practice, we then get an estimate of the NLL by numerically estimating the integral in \cref{eqn:loglikelihood}, starting from a clean sample $\vx_{0} \sim p_{\text{data}}$ drawn from a hold-out validation set, from the smallest noise level in the noise schedule $\sigma_{\text{min}}^{2}$ to the biggest $\sigma_{\text{max}}^{2}$ and evaluating $-\log p({\vx_{0}})$.

For PIDM (which uses an $\vx_{0}$ prediction variance-preserving DDPM diffusion model), we can still use \cref{eqn:loglikelihood}, but we must take some adjustments into account. Denoting the output of the DDPM denoiser by $\hat{\vx}_{0}$

\begin{itemize}
    \item The input clean sample $\vx_{0}$ to the DDPM denoiser at each integration step must be rescaled by $\frac{1}{\sqrt{ \sigma_{t}^{2}+1} }$: we denote this input on the forward noising process $\vx_{t}$. For this denoiser, the noise level at each $t$ is given by $\sigma_{t} = \frac{\sqrt{1-\alpha_{t}}}{\sqrt{\alpha_{t}}}$, with $\alpha_{t}$ the noise schedule for VP-DDPM.
    \item Then, the drift term at each step will be given by $g(t) = \frac{\vx_{t} - \sqrt{\alpha_{t}}\hat{\vx_{0}} } {\sqrt{1-\alpha_{t}}}$. We can use this drift in \cref{eqn:loglikelihood} and proceed with numerical integration as before.
\end{itemize}

\section{Darcy Flow}
\label{app:DarcyFlowWriteup}

\emph{Darcy Flow} refers to a set of equations used to study the (laminar) flow of fluids in porous media. In two dimensions it is defined by the following set of equations for $\bm{x}=(x,y)$ \cite{schlichting_onset_2017, jacobsen_cocogen_2025, bastek_physics-informed_2025}:
\begin{align}
    \bm{u}(\bm{x}) =& -K(\bm{x})\nabla p(\bm{x})\\
    \nabla \cdot \bm{u} =& f(\bm{x})\\
    \bm{u} \cdot \hat{\bm{n}} =& 0, \text{ boundary condition}\\
    \int p(\bm{x}) \odif{\bm{x}} =& 0,
\end{align}
where $K$ is a permeability field that describes how easy a fluid flows through the medium, $p$ is the pressure field that defines where the fluid is pushed and pulled, $\bm{u}$ the velocity field of the fluid (as visualized in \cref{tab:misspec-residuals}) and $f$ is the net flow of fluid through a given point. \emph{Net} flow means that if there is the same amount fluid entering and exiting at a given point, then the net flow is zero. As a more concrete example, we can use Darcy flow to describe how water will flow through a body of sand. We can expect more water to flow at the areas where we apply more pressure to squeeze out the water.

There are many important assumptions around Darcy flow before it is applicable. We must have laminar flow at steady state, which means that there must not be sudden changes in the flow of the fluid. This depends on the geometry of the body that the fluid is traversing through, its viscosity, its speed, among other factors that are generally described through the Reynolds number \citet{schlichting_onset_2017, potter_fluidmech_2025}. The Reynolds number needed to break into critical flow also will change between different configurations, so it may be difficult to know a priori if Darcy flow is an appropriate way to model the system.

Through experience it has been observed that the Darcy flow approximation can work well for media that are ``fine-grained'' enough, as the gaps between the particles in the porous matrix do not break the flow of water \citet{woessner2020hydrogeologic_ch4}. However, at a big enough particle size, water starts to bounce more between collisions, causing the flow to become more turbulent.

\subsection{Qualitative results from Darcy Flow} \label{app:QualsDarcy}

\cref{fig:biasDist} shows a histogram of the learned distribution of values for pressure and permeability using each of the compared methods. Particularly, as noted by \citep{bastek_physics-informed_2025}, PIDM presents excessive bias compared to the other methods.

\begin{figure}[ht!]
    \centering
    \vspace{-1em}
    \includegraphics[width=\linewidth]{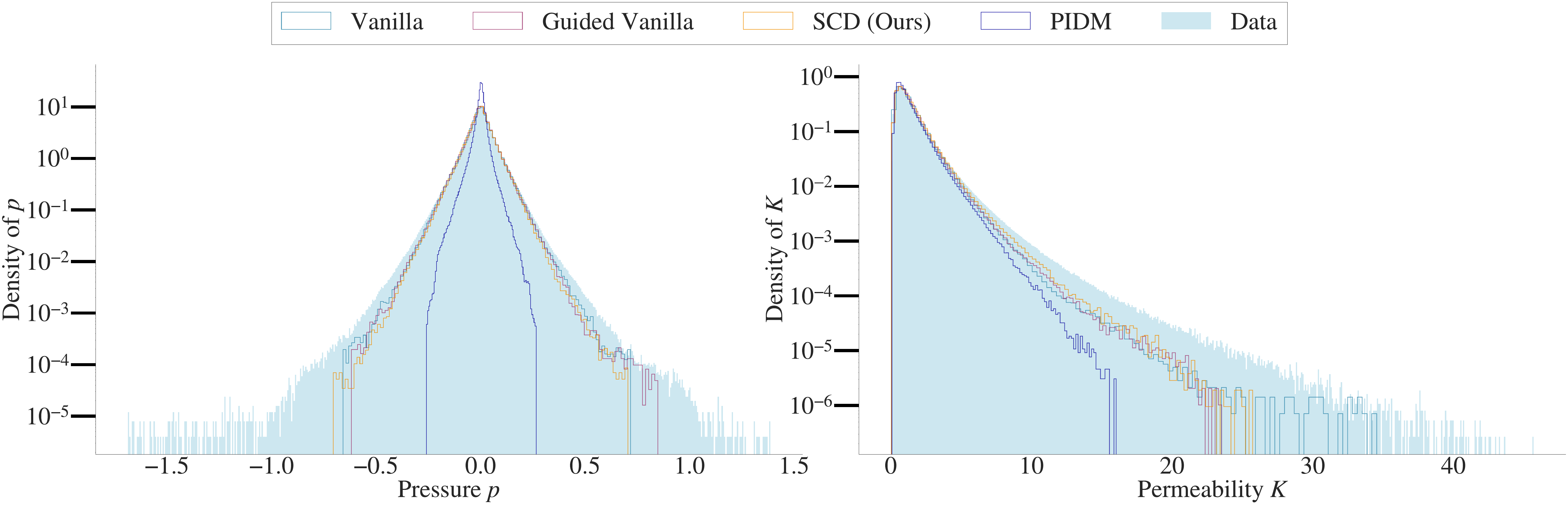}
    \caption{Darcy flow validation. 1000 samples were generated with each approach, using 100 denoising steps for each. Error bars indicate the minimum and maximum from the samples. The distribution PIDM learns is highly biased; all other methods have higher distributional fidelity. Particularly, our method shows low distributional bias compared to vanilla. 
    }
    \label{fig:biasDist}
    \vspace{-6pt}
\end{figure}

\section{Architectures and training details}
\label{app:training}

The architectures and experiments on this work were implemented using PyTorch 2.4.1. All experiments were ran on single GPUs, from either NVIDIA H200, H100, A100 or V100 GPUs. For both experiments, $\gamma_{\theta}$ is implemented as a 2-layer MLP with an embedding dimension of 100. The input for either task is the denoiser output $D_{\theta}(\bm{x_{t}},t)$ and the diffusion time $t$. Given that these differential equations generally have high condition numbers, meaning that the constraint gradient can easily blow up at high noise levels, we also also may pre-condition $\gamma_{\theta}$: we tried both by explicitly dividing by the square norm of the gradient by the result in \cref{prop:gamma_role}, and separately with an exponential scaler $\exp(-\lambda t)$ following the theory by \citet{gao2026c2fgcontrolclassifierfreeguidance}; we saw similar performance with both approaches, but better training stability with the exponential scaler, so we use the latter for our experiments. As a simple heuristic to tune $\lambda$ we take the mean residual value of the output of a denoiser at initialization, which is roughly white Gaussian noise with standard deviation $\sigma_{\text{data}}$, and then tune $\lambda$ so that at $t\approx \sigma_{\text{data}}$ the constraint gradient term has a norm of approximately $\sigma_{\text{data}}$, following the practices by \citet{karras_analyzing_2023} to keep input-output magnitudes in similar ranges.

\subsection{Modified loss function}
\label{app:modifiedLoss}

The loss function by \citet{karras_elucidating_2022} has the form seen in \cref{eq:dsm} with the choice of distributions $t \sim\ \text{LogNormal}(\mu_{\text{train}}, \sigma_{\text{train}}^2 ) $ and $\bm{x}_{t} \sim\ \mathcal{N}(\bm{x}_{0}, t\mathbf{I}) $, and weighting function $ w(t) = (t + \sigma_{\text{data}^2})/ (t\sigma_{\text{data}}^{2} )$ . Crucially, the choice for the distribution for $t$ is free, i.e. a design choice; \citet{karras_elucidating_2022} choose the log-normal distribution based on the observation that most of the important denoising steps in a diffusion model happen in the ``intermediate'' noise levels, since at high noise levels there is little distinction between steps and at very-low noise levels the differences are negligible. However, on the ODE problems presented in this paper we have observed that low noise levels have a substantial effect on the final residuals. Simply trying to set $\mu_{\text{train}}$ to a small value can cause numerical stability issues, because it may start sampling values that are too small and consequently make the weighting function blow up.

Following this, and based on the observation that in practice most numerical integration samplers always end at a predetermined minimum time step, we choose the noise level distribution $t \sim \text{TruncLogNormal}(\mu_{\text{train}}, \sigma^2_{\text{train}}, a)$ where $a$ defines the lowest possible noise level to be sampled. With this change, we noticed a substantial improvement in the residuals in the Darcy Flow experiment, with the results shown in \cref{tab:modifiedLossResiduals}. We use a mean of -1.5 and standard deviation of 1.2 for the log-normal loss and a mean of -2, standard deviation of 1.7 and truncation lower limit of -4 for the truncated log-normal loss.

\begin{table}[ht!]
  \centering\footnotesize
  \caption{Residuals obtained in the Darcy flow experiments using the noise level distribution $t \sim\ \text{LogNormal}(\mu_{\text{train}}, \sigma_{\text{train}}^2 ) $ by \citet{karras_elucidating_2022} and our choice $t \sim \text{TruncLogNormal}(\mu_{\text{train}}, \sigma^2_{\text{train}}, a)$. Using the truncated log-normal shows a substantial improvement over the log-normal in this task.}\vspace{-2pt}
  \label{tab:modifiedLossResiduals}
  \setlength{\tabcolsep}{7pt}
\begin{tabular}{llcccc}
\toprule
& & \em Original & \multicolumn{3}{l}{\tikz[baseline=0pt]\draw[anchor=base,black,->] (0,0) --node[yshift=2pt,font=\scriptsize\em]{increasing misspecification} ++(7cm,0);} \\
\bf Distribution &
\bf Method & 
\bf $f_\mathrm{max} = 10$ & 
\bf $f_\mathrm{max} = 20$ &
\bf $f_\mathrm{max} = 30$ & 
\bf $f_\mathrm{max} = 40$ \\ 
\midrule
& Vanilla & $0.874 \pm 0.411$ & $0.874 \pm 0.411$ & $0.874 \pm 0.411$ & $\underline{0.874 \pm 0.411}$ \\
log-normal & Guided Vanilla & $0.869 \pm 0.382$ & $0.862 \pm 0.372$ & $0.872 \pm 0.404$ & $0.869 \pm 0.379$ \\
& SCD (ours) & $0.342 \pm 0.136$ & $0.414 \pm 0.174$ & $0.429 \pm 0.173$ & $0.428 \pm 0.174$ \\
\midrule
& Vanilla & $0.157 \pm 0.071$ & $0.157 \pm 0.071$ & $0.157 \pm 0.071$ & $0.157 \pm 0.071$ \\
Truncated & Guided Vanilla & $0.139 \pm 0.067$ & $0.141 \pm 0.061$ & $0.140 \pm 0.065$ & $0.139 \pm 0.066$ \\
log-normal & SCD (ours) & $0.106 \pm 0.049$ & $0.113 \pm 0.058$ & $0.114 \pm 0.059$ & $0.118 \pm 0.051$ \\
\bottomrule
\end{tabular}
% \vspace{-3pt}
\end{table}

Following these results, we use the truncated log-normal distribution for all our experiments.

\newpage
\subsection{Circles}
\label{appdx:circlesTraining}

For the toy example we use a 3 layer MLP with an embedding dimension of 128. To train the networks we use the Adam optimizer with $\beta_{1}=0.9$, $\beta_{2}=0.999$ and a fixed learning rate of $\alpha=1\cdot 10^{-4}$. 

The data set consisted of 10000 points sampled from the unit circle, and the models were trained over 1000 epochs with a batch size of 128.

For the ``Dent'' variant of misspecification, we use the following parametric curve:

\begin{align}
    C(\theta) =& (r(\theta)\cos(\theta), r(\theta) \sin(\theta) ) \\
    r(\theta) =& 1 - 0.25\cdot \beta\left( \frac{\text{wrap} (\theta - \frac{\pi}{2}) }{1.2}, 5 \right) \cdot \left( 1 + 0.6\left( 1 - 2\left( \frac{\text{wrap}(\theta - \frac{\pi}{2} )}{1.2} \right)^{2} \right) \right)\\
    \text{wrap}(\theta) &= \left( \left( \theta + \pi \right)\mod 2\pi \right) - \pi\\
    \beta\left( u, 5 \right) &= \begin{cases}
        (1-u^{2})^{5}: |u| < 1\\
        0 \text{ otherwise},
    \end{cases}
\end{align}

where C defines the coordinates of every point in the curve.

\subsection{Darcy Flow}
\label{appdx:darcyTraining}

For Darcy Flow we used the UNet implementation by \citet{karras_analyzing_2023}. We use the Heun sampler implementation by \citet{karras_elucidating_2022} with 100 denoising steps. The used hyperparameters on the network are summarized in \cref{tab:hyperparams}.

\begin{table}[ht]
    \centering
    \caption{Architecture hyperparameters for the Darcy Flow experiments}
    \begin{tabular}{c|c}
    \toprule
        Hyperparameter & Value \\ \midrule
        Model channels & 24 \\
        Number of residual blocks & 8 \\
        Per-resolution multipliers & [1, 2, 3, 4] \\
        Attention resolutions & [16, 8] \\ \bottomrule
    \end{tabular}
    \label{tab:hyperparams}
\end{table}

\subsection{Helmholtz Equation}
\label{appdx:helmholtzTraining}

For the Helmholtz Equation we used the UNet implementation by \citet{karras_analyzing_2023}. We use the Heun sampler implementation by \citet{karras_elucidating_2022} with 100 denoising steps. The used hyperparameters on the network are summarized in \cref{tab:helmholtzhyperparams}.

\begin{table}[ht]
    \centering
    \caption{Architecture hyperparameters for the Helmholtz Equation experiments}
    \begin{tabular}{c|c}
    \toprule
        Hyperparameter & Value \\ \midrule
        Model channels & 32 \\
        Number of residual blocks & 8 \\
        Per-resolution multipliers & [1, 2, 3, 4] \\
        Attention resolutions & [32, 16] \\ \bottomrule
    \end{tabular}
    \label{tab:helmholtzhyperparams}
\end{table}

\subsection{Runtimes}

The details on the runtimes on an NVIDIA H200 GPU with vanilla diffusion and our method are shown in \cref{tab:runtimes}. We note that our method sees the most impact at sampling time, since the overhead duplicates per each sampling iteration because the Heun sampler makes two neural function evaluations per iteration.

\begin{table}[ht]
  \centering\footnotesize
  \caption{ Runtimes of vanilla and our method on training and sampling on an NVIDIA H200 GPU. Sampling is done for eight samples at a time using the Heun sampler implementation by \citet{karras_elucidating_2022}. Numbers are reported as a mean with two standard errors over 300 iterations. Units of time are specified for each column. }
  \label{tab:runtimes}
  \setlength{\tabcolsep}{7pt}
\begin{tabular}{lcc}
\toprule
\multicolumn{1}{l}{\bf Method}  & \multicolumn{1}{c}{\bf Training iteration wall clock time (ms)}  & \multicolumn{1}{c}{\bf 100-step sampling wall clock time (s)} 
\\ \midrule
Vanilla & $ 110 \pm 3.62 $ & $ 7.79 \pm 0.021 $ \\
SCD (ours) & $ 148 \pm 2.25 $ & $ 10.3 \pm 0.029 $ \\
\bottomrule
\end{tabular}
\end{table}

The baseline PIDM results were reproduced with a pre-trained model provided by \citet{bastek_physics-informed_2025}, and for the misspecification experiments we trained a diffusion model with PIDM with each misspecified residual using the same hyperparameter setups as \citet{bastek_physics-informed_2025}. The vanilla score-matching and SCD networks were trained using Adam with an initial learning rate of $2\cdot 10^{-3}$ and weight decay as detailed by \citet{karras_analyzing_2023}. Both the vanilla model and SCD were trained for 300k iterations, although SCD plateaus at around 250k. We use a batch size of 64. For all experiments we use Exponentially Moving Averages for sample generation with a decay rate of 0.99.

\section{Additional ``chop'' samples}

\begin{figure}[ht]
\begin{subfigure}[b]{.32\textwidth}
  \centering
  \includegraphics[width=.85\linewidth,trim=0 15 0 15,clip]{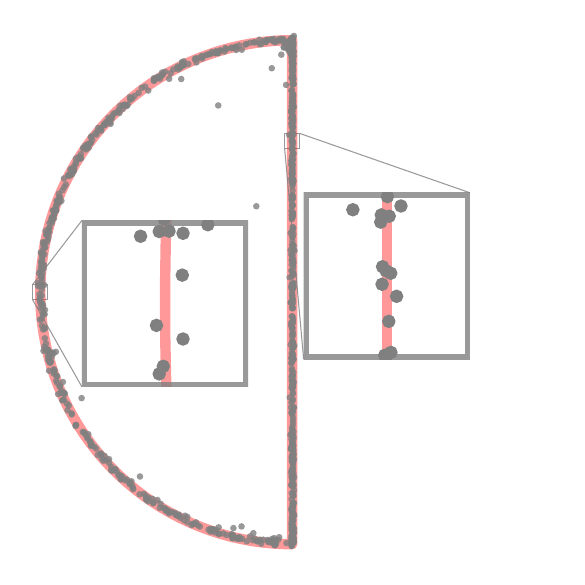}
  \caption{Vanilla diffusion}
\end{subfigure}
\hfill
\begin{subfigure}[b]{.32\textwidth}
  \centering
  \includegraphics[width=.85\linewidth,trim=0 15 0 15,clip]{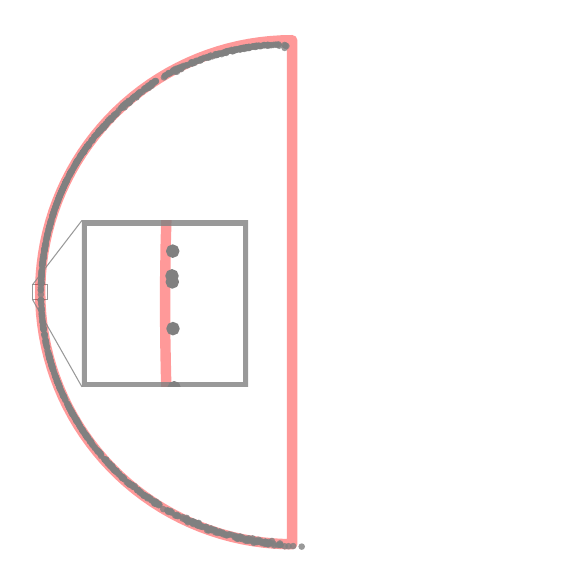}
  \caption{PIDM (baseline)}
\end{subfigure}
\hfill
\begin{subfigure}[b]{.32\textwidth}
  \centering
  \includegraphics[width=.85\linewidth,trim=0 15 0 15,clip]{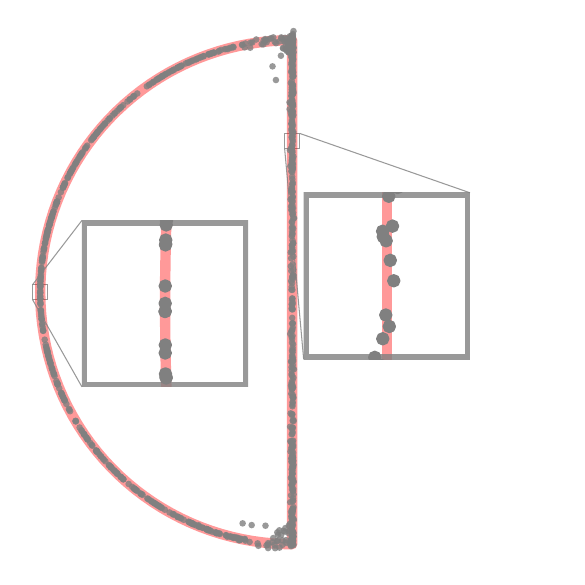}
  \caption{SCD (ours)}
\end{subfigure}
\caption{Additional samples from the ``chop'' example. }
\label{fig:chops}
\end{figure}

\end{document}

%% file: math_commands.tex
\usepackage{amsmath,amsfonts,bm}

\def\eqref#1{equation~\ref{#1}}
\def\1{\bm{1}}

\def\eps{{\epsilon}}

\def\vtheta{{\bm{\theta}}}

\def\vx{{\bm{x}}}

\newcommand{\vepsilon}{\boldsymbol{\epsilon}}

\DeclareMathAlphabet{\mathsfit}{\encodingdefault}{\sfdefault}{m}{sl}
\SetMathAlphabet{\mathsfit}{bold}{\encodingdefault}{\sfdefault}{bx}{n}

%% file: main.bbl
\begin{thebibliography}{63}
\providecommand{\natexlab}[1]{#1}
\providecommand{\url}[1]{\texttt{#1}}
\expandafter\ifx\csname urlstyle\endcsname\relax
  \providecommand{\doi}[1]{doi: #1}\else
  \providecommand{\doi}{doi: \begingroup \urlstyle{rm}\Url}\fi

\bibitem[Baldan et~al.(2025)Baldan, Liu, Guardone, and
  Thuerey]{baldan_flow_2025}
Baldan, G., Liu, Q., Guardone, A., and Thuerey, N.
\newblock Flow {Matching} {Meets} {PDEs}: {A} {Unified} {Framework} for
  {Physics}-{Constrained} {Generation}, 2025.
\newblock preprint: arXiv:2506.08604.

\bibitem[Bastek et~al.(2025)Bastek, Sun, and
  Kochmann]{bastek_physics-informed_2025}
Bastek, J.-H., Sun, W., and Kochmann, D.
\newblock Physics-informed diffusion models.
\newblock In \emph{International Conference on Learning Representations
  (ICLR)}, 2025.

\bibitem[Ben-Hamu et~al.(2024)Ben-Hamu, Puny, Gat, Karrer, Singer, and
  Lipman]{benhamu_dflow}
Ben-Hamu, H., Puny, O., Gat, I., Karrer, B., Singer, U., and Lipman, Y.
\newblock D-flow: Differentiating through flows for controlled generation.
\newblock In \emph{Forty-first International Conference on Machine Learning
  (ICML)}, 2024.

\bibitem[Boys et~al.(2024)Boys, Girolami, Pidstrigach, Reich, Mosca, and
  Akyildiz]{boys_tweedie_2024}
Boys, B., Girolami, M., Pidstrigach, J., Reich, S., Mosca, A., and Akyildiz,
  O.~D.
\newblock Tweedie {Moment} {Projected} {Diffusions} for {Inverse} {Problems}.
\newblock \emph{Transactions on Machine Learning Research}, 2024.

\bibitem[Chen \& Lipman(2024)Chen and Lipman]{chenflow}
Chen, R.~T. and Lipman, Y.
\newblock Flow matching on general geometries.
\newblock In \emph{International Conference on Learning Representations
  (ICLR)}, 2024.

\bibitem[Cheng et~al.(2025)Cheng, Han, Maddix, Ansari, Stuart, Mahoney, and
  Wang]{chenggradient}
Cheng, C., Han, B., Maddix, D.~C., Ansari, A.~F., Stuart, A., Mahoney, M.~W.,
  and Wang, B.
\newblock Gradient-free generation for hard-constrained systems.
\newblock In \emph{The Thirteenth International Conference on Learning
  Representations}, 2025.

\bibitem[Christopher et~al.(2024)Christopher, Baek, and
  Fioretto]{Christopher2024ConstrainedSW}
Christopher, J.~K., Baek, S., and Fioretto, F.
\newblock Constrained synthesis with projected diffusion models.
\newblock \emph{Advances in Neural Information Processing Systems 37}, 2024.
\newblock URL \url{https://api.semanticscholar.org/CorpusID:267500184}.

\bibitem[Christopher et~al.(2026)Christopher, Warner, and
  Fioretto]{christopher2026constraintawareflowmatchingdecision}
Christopher, J.~K., Warner, J.~E., and Fioretto, F.
\newblock Constraint-aware flow matching: Decision aligned end-to-end training
  for constrained sampling, 2026.
\newblock URL \url{https://arxiv.org/abs/2605.12754}.

\bibitem[Chung et~al.(2022)Chung, Sim, Ryu, and Ye]{chung2022improving}
Chung, H., Sim, B., Ryu, D., and Ye, J.~C.
\newblock Improving diffusion models for inverse problems using manifold
  constraints.
\newblock In \emph{Advances in Neural Information Processing Systems 35
  (NeurIPS)}, pp.\  25683--25696, 2022.

\bibitem[Chung et~al.(2023{\natexlab{a}})Chung, Kim, Mccann, Klasky, and
  Ye]{chung_diffusion_2023}
Chung, H., Kim, J., Mccann, M.~T., Klasky, M.~L., and Ye, J.~C.
\newblock Diffusion posterior sampling for general noisy inverse problems.
\newblock In \emph{International Conference on Learning Representations
  (ICLR)}, 2023{\natexlab{a}}.

\bibitem[Chung et~al.(2023{\natexlab{b}})Chung, Kim, and Ye]{chung2023direct}
Chung, H., Kim, J., and Ye, J.~C.
\newblock Direct diffusion bridge using data consistency for inverse problems.
\newblock In \emph{Advances in Neural Information Processing Systems 36
  (NeurIPS)}, pp.\  7158--7169, 2023{\natexlab{b}}.

\bibitem[De~Bortoli et~al.(2022)De~Bortoli, Mathieu, Hutchinson, Thornton, Teh,
  and Doucet]{de2022riemannian}
De~Bortoli, V., Mathieu, E., Hutchinson, M., Thornton, J., Teh, Y.~W., and
  Doucet, A.
\newblock Riemannian score-based generative modelling.
\newblock In \emph{Advances in Neural Information Processing Systems 35
  (NeurIPS)}, pp.\  2406--2422, 2022.

\bibitem[Denker et~al.(2024)Denker, Vargas, Padhy, Didi, Mathis, Barbano,
  Dutordoir, Mathieu, Komorowska, and Lio]{denker2024deft}
Denker, A., Vargas, F., Padhy, S., Didi, K., Mathis, S., Barbano, R.,
  Dutordoir, V., Mathieu, E., Komorowska, U.~J., and Lio, P.
\newblock {DEFT}: {E}fficient fine-tuning of diffusion models by learning the
  generalised $h$-transform.
\newblock In \emph{Advances in Neural Information Processing Systems 37
  (NeurIPS)}, pp.\  19636--19682, 2024.

\bibitem[Finzi et~al.(2021)Finzi, Benton, and Wilson]{finzi_soft_2021}
Finzi, M., Benton, G., and Wilson, A.~G.
\newblock Residual pathway priors for soft equivariance constraints.
\newblock In \emph{Advances in Neural Information Processing Systems
  (NeurIPS)}, pp.\  30037--30049. Curran Associates, Inc., 2021.

\bibitem[Fishman et~al.(2023{\natexlab{a}})Fishman, Klarner, De~Bortoli,
  Mathieu, and Hutchinson]{fishmandiffusion}
Fishman, N., Klarner, L., De~Bortoli, V., Mathieu, E., and Hutchinson, M.~J.
\newblock Diffusion models for constrained domains.
\newblock \emph{Transactions on Machine Learning Research}, 2023{\natexlab{a}}.

\bibitem[Fishman et~al.(2023{\natexlab{b}})Fishman, Klarner, Mathieu,
  Hutchinson, and De~Bortoli]{fishman2023metropolis}
Fishman, N., Klarner, L., Mathieu, E., Hutchinson, M., and De~Bortoli, V.
\newblock Metropolis sampling for constrained diffusion models.
\newblock In \emph{Advances in Neural Information Processing Systems 36
  (NeurIPS)}, pp.\  62296--62331, 2023{\natexlab{b}}.

\bibitem[Gao et~al.(2026)Gao, Zheng, Zou, Yang, Liu, Fan, Zhang, Zhang, Chen,
  Jiang, Li, and Wang]{gao2026c2fgcontrolclassifierfreeguidance}
Gao, J., Zheng, T., Zou, J., Yang, F., Liu, S., Fan, L., Zhang, Z., Zhang, H.,
  Chen, J., Jiang, P.-T., Li, B., and Wang, J.
\newblock C2fg: Control classifier-free guidance via score discrepancy
  analysis, 2026.
\newblock URL \url{https://arxiv.org/abs/2603.08155}.

\bibitem[Ho et~al.(2020)Ho, Jain, and Abbeel]{ho_denoising_2020}
Ho, J., Jain, A., and Abbeel, P.
\newblock Denoising diffusion probabilistic models.
\newblock In \emph{Advances in Neural Information Processing Systems
  (NeurIPS)}, pp.\  6840--6851. Curran Associates Inc., 2020.

\bibitem[Ho et~al.(2022)Ho, Salimans, Gritsenko, Chan, Norouzi, and
  Fleet]{ho_video_2022}
Ho, J., Salimans, T., Gritsenko, A., Chan, W., Norouzi, M., and Fleet, D.~J.
\newblock Video {Diffusion} {Models}.
\newblock In \emph{Advances in Neural Information Processing Systems 35
  (NeurIPS)}, pp.\  8633--8646, 2022.

\bibitem[Hu et~al.(2025)Hu, Wang, Zheng, Zhang, Feng, Feng, Wei, Wang, Ma, and
  Wu]{huwavelet}
Hu, P., Wang, R., Zheng, X., Zhang, T., Feng, H., Feng, R., Wei, L., Wang, Y.,
  Ma, Z.-M., and Wu, T.
\newblock Wavelet diffusion neural operator.
\newblock In \emph{The Thirteenth International Conference on Learning
  Representations}, 2025.

\bibitem[Huang et~al.(2022)Huang, Aghajohari, Bose, Panangaden, and
  Courville]{huang2022riemannian}
Huang, C.-W., Aghajohari, M., Bose, J., Panangaden, P., and Courville, A.~C.
\newblock Riemannian diffusion models.
\newblock In \emph{Advances in Neural Information Processing Systems 35
  (NeurIPS)}, pp.\  2750--2761, 2022.

\bibitem[Huang et~al.(2024)Huang, Yang, Wang, and Park]{huang2024diffusionpde}
Huang, J., Yang, G., Wang, Z., and Park, J.~J.
\newblock Diffusionpde: Generative pde-solving under partial observation.
\newblock In \emph{Advances in Neural Information Processing Systems 37
  (NeurIPS)}, pp.\  130291--130323, 2024.

\bibitem[Jacobsen et~al.(2025)Jacobsen, Zhuang, and
  Duraisamy]{jacobsen_cocogen_2025}
Jacobsen, C., Zhuang, Y., and Duraisamy, K.
\newblock {CoCoGen}: {P}hysically consistent and conditioned score-based
  generative models for forward and inverse problems.
\newblock \emph{SIAM Journal on Scientific Computing}, 47\penalty0
  (2):\penalty0 C399--C425, 2025.
\newblock \doi{10.1137/24M1636071}.

\bibitem[Karras et~al.(2022)Karras, Aittala, Aila, and
  Laine]{karras_elucidating_2022}
Karras, T., Aittala, M., Aila, T., and Laine, S.
\newblock Elucidating the design space of diffusion based generative models.
\newblock In \emph{Advances in Neural Information Processing Systems
  (NeurIPS)}, pp.\  26565--26577. Curran Associates, Inc., 2022.

\bibitem[Karras et~al.(2024)Karras, Aittala, Lehtinen, Hellsten, Aila, and
  Laine]{karras_analyzing_2023}
Karras, T., Aittala, M., Lehtinen, J., Hellsten, J., Aila, T., and Laine, S.
\newblock Analyzing and {Improving} the {Training} {Dynamics} of {Diffusion}
  {Models}.
\newblock \emph{2024 IEEE/CVF Conference on Computer Vision and Pattern
  Recognition (CVPR)}, pp.\  24174--24184, 2024.

\bibitem[Khalafi et~al.(2024)Khalafi, Ding, and
  Ribeiro]{khalafi2024constrained}
Khalafi, S., Ding, D., and Ribeiro, A.
\newblock Constrained diffusion models via dual training.
\newblock In \emph{Advances in Neural Information Processing Systems 37
  (NeurIPS)}, pp.\  26543--26576, 2024.

\bibitem[Khalafi et~al.(2025)Khalafi, Hounie, Ding, and
  Ribeiro]{khalaficomposition}
Khalafi, S., Hounie, I., Ding, D., and Ribeiro, A.
\newblock Composition and alignment of diffusion models using constrained
  learning.
\newblock In \emph{2nd Workshop on Models of Human Feedback for AI Alignment},
  2025.

\bibitem[Kingma et~al.(2021)Kingma, Salimans, Poole, and
  Ho]{kingma2021variational}
Kingma, D., Salimans, T., Poole, B., and Ho, J.
\newblock Variational diffusion models.
\newblock \emph{Advances in neural information processing systems},
  34:\penalty0 21696--21707, 2021.

\bibitem[Kong et~al.(2023)Kong, Brekelmans, and Steeg]{kong_infotheo_2023}
Kong, X., Brekelmans, R., and Steeg, G.~V.
\newblock Information-theoretic diffusion.
\newblock In \emph{The Eleventh International Conference on Learning
  Representations, {ICLR} 2023, Kigali, Rwanda, May 1-5, 2023}. OpenReview.net,
  2023.
\newblock URL \url{https://openreview.net/forum?id=UvmDCdSPDOW}.

\bibitem[Krishnapriyan et~al.(2021)Krishnapriyan, Gholami, Zhe, Kirby, and
  Mahoney]{Krishnapriyan_characterizing_2021}
Krishnapriyan, A., Gholami, A., Zhe, S., Kirby, R., and Mahoney, M.~W.
\newblock Characterizing possible failure modes in physics-informed neural
  networks.
\newblock In \emph{Advances in Neural Information Processing Systems
  (NeurIPS)}, pp.\  26548--26560. Curran Associates, Inc., 2021.

\bibitem[Liang et~al.(2026)Liang, Naderiparizi, Liu, Zwartsenberg, and
  Wood]{liang2026improvedconstrainedgenerationbridging}
Liang, X., Naderiparizi, S., Liu, Y., Zwartsenberg, B., and Wood, F.
\newblock Improved constrained generation by bridging pretrained generative
  models, 2026.
\newblock URL \url{https://arxiv.org/abs/2603.06742}.

\bibitem[Liu et~al.(2023{\natexlab{a}})Liu, Chen, Theodorou, and
  Tao]{liu_mirror_2023}
Liu, G.-H., Chen, T., Theodorou, E., and Tao, M.
\newblock Mirror diffusion models for constrained and watermarked generation.
\newblock In \emph{Advances in Neural Information Processing Systems 36
  (NeurIPS)}, pp.\  42898--42917, 2023{\natexlab{a}}.

\bibitem[Liu et~al.(2023{\natexlab{b}})Liu, Vahdat, Huang, Theodorou, Nie, and
  Anandkumar]{liu20232}
Liu, G.-H., Vahdat, A., Huang, D.-A., Theodorou, E., Nie, W., and Anandkumar,
  A.
\newblock I$^2$sb: Image-to-image {S}chr{\"o}dinger bridge.
\newblock In \emph{International Conference on Machine Learning (ICML)}, pp.\
  22042--22062. PMLR, 2023{\natexlab{b}}.

\bibitem[Liu et~al.(2023{\natexlab{c}})Liu, Wu, Ye, and Liu]{Liu2023LearningDB}
Liu, X., Wu, L., Ye, M., and Liu, Q.
\newblock Learning diffusion bridges on constrained domains.
\newblock In \emph{International Conference on Learning Representations},
  2023{\natexlab{c}}.
\newblock URL \url{https://api.semanticscholar.org/CorpusID:259298235}.

\bibitem[Lou \& Ermon(2023)Lou and Ermon]{lou_reflected_2023}
Lou, A. and Ermon, S.
\newblock Reflected diffusion models.
\newblock In \emph{International {Conference} on {Machine} {Learning}}, pp.\
  22675--22701. PMLR, 2023.

\bibitem[Mardani et~al.(2024)Mardani, Song, Kautz, and
  Vahdat]{mardanivariational}
Mardani, M., Song, J., Kautz, J., and Vahdat, A.
\newblock A variational perspective on solving inverse problems with diffusion
  models.
\newblock In \emph{International Conference on Learning Representations
  (ICLR)}, 2024.

\bibitem[Mildenhall et~al.(2021)Mildenhall, Srinivasan, Tancik, Barron,
  Ramamoorthi, and Ng]{mildenhall2021nerf}
Mildenhall, B., Srinivasan, P.~P., Tancik, M., Barron, J.~T., Ramamoorthi, R.,
  and Ng, R.
\newblock Nerf: Representing scenes as neural radiance fields for view
  synthesis.
\newblock \emph{Communications of the ACM}, 65\penalty0 (1):\penalty0 99--106,
  2021.

\bibitem[Naderiparizi et~al.(2025)Naderiparizi, Liang, Zwartsenberg, and
  Wood]{saeid2025ManuallyBridged}
Naderiparizi, S., Liang, X., Zwartsenberg, B., and Wood, F.
\newblock Constrained generative modeling with manually bridged diffusion
  models.
\newblock In \emph{Proceedings of the Thirty-Ninth AAAI Conference on
  Artificial Intelligence and Thirty-Seventh Conference on Innovative
  Applications of Artificial Intelligence and Fifteenth Symposium on
  Educational Advances in Artificial Intelligence}, AAAI'25/IAAI'25/EAAI'25.
  AAAI Press, 2025.
\newblock ISBN 978-1-57735-897-8.
\newblock \doi{10.1609/aaai.v39i18.34159}.
\newblock URL \url{https://doi.org/10.1609/aaai.v39i18.34159}.

\bibitem[Okhotin et~al.(2023)Okhotin, Molchanov, Arkhipkin, Bartosh, Ohanesian,
  Alanov, and Vetrov]{okhotin2023star}
Okhotin, A., Molchanov, D., Arkhipkin, V., Bartosh, G., Ohanesian, V., Alanov,
  A., and Vetrov, D.~P.
\newblock Star-shaped denoising diffusion probabilistic models.
\newblock In \emph{Advances in Neural Information Processing Systems 36
  (NeurIPS)}, pp.\  10038--10067, 2023.

\bibitem[Oommen et~al.(2024)Oommen, Bora, Zhang, and
  Karniadakis]{oommen2024integrating}
Oommen, V., Bora, A., Zhang, Z., and Karniadakis, G.~E.
\newblock Integrating neural operators with diffusion models improves spectral
  representation in turbulence modeling.
\newblock \emph{arXiv preprint arXiv:2409.08477}, 2024.

\bibitem[Peng et~al.(2024)Peng, Zheng, Dai, Xiao, Li, Zou, and
  Xiong]{peng_improving_2024}
Peng, X., Zheng, Z., Dai, W., Xiao, N., Li, C., Zou, J., and Xiong, H.
\newblock Improving {Diffusion} {Models} for {Inverse} {Problems} {Using}
  {Optimal} {Posterior} {Covariance}.
\newblock In \emph{Proceedings of the International Conference on Machine
  Learning (ICML)}, 2024.

\bibitem[Poole et~al.(2023)Poole, Jain, Barron, and
  Mildenhall]{pooledreamfusion}
Poole, B., Jain, A., Barron, J.~T., and Mildenhall, B.
\newblock {DreamFusion}: Text-to-{3D} using {2D} diffusion.
\newblock In \emph{International Conference on Learning Representations
  (ICLR)}, 2023.

\bibitem[Potter \& Ramadan(2012)Potter and Ramadan]{potter_fluidmech_2025}
Potter, M.~C. and Ramadan, B.~H.
\newblock \emph{An Introduction to Fluid Mechanics}.
\newblock Springer Cham, 4 edition, 2012.

\bibitem[Rathore et~al.(2024)Rathore, Lei, Frangella, Lu, and
  Udell]{rathore_challenges_2024}
Rathore, P., Lei, W., Frangella, Z., Lu, L., and Udell, M.
\newblock Challenges in training {PINN}s: A loss landscape perspective.
\newblock In \emph{Proceedings of the 41st International Conference on Machine
  Learning (ICML)}, 2024.

\bibitem[Rissanen et~al.(2025)Rissanen, Heinonen, and
  Solin]{rissanen_free_2025}
Rissanen, S., Heinonen, M., and Solin, A.
\newblock Free hunch: {D}enoiser covariance estimation for diffusion models
  without extra costs.
\newblock In \emph{International Conference on Learning Representations
  (ICLR)}, 2025.

\bibitem[Schlichting \& Gersten(2017)Schlichting and
  Gersten]{schlichting_onset_2017}
Schlichting, H. and Gersten, K.
\newblock Onset of turbulence (stability theory).
\newblock In \emph{Boundary-Layer Theory}, chapter~15, pp.\  415--419.
  Springer, ninth edition, 2017.

\bibitem[Shysheya et~al.(2024)Shysheya, Diaconu, Bergamin, Perdikaris,
  Hern{\'a}ndez-Lobato, Turner, and Mathieu]{shysheya2024conditional}
Shysheya, A., Diaconu, C., Bergamin, F., Perdikaris, P., Hern{\'a}ndez-Lobato,
  J.~M., Turner, R., and Mathieu, E.
\newblock On conditional diffusion models for pde simulations.
\newblock \emph{Advances in Neural Information Processing Systems},
  37:\penalty0 23246--23300, 2024.

\bibitem[Sohl-Dickstein et~al.(2015)Sohl-Dickstein, Weiss, Maheswaranathan, and
  Ganguli]{sohl-dickstein_deep_2015}
Sohl-Dickstein, J., Weiss, E., Maheswaranathan, N., and Ganguli, S.
\newblock Deep unsupervised learning using nonequilibrium thermodynamics.
\newblock In \emph{Proceedings of the 32nd International Conference on Machine
  Learning (ICML)}, volume~37 of \emph{Proceedings of {Machine} {Learning}
  {Research}}, pp.\  2256--2265, Lille, France, 2015. PMLR.

\bibitem[Song et~al.(2023{\natexlab{a}})Song, Vahdat, Mardani, and
  Kautz]{song_pseudoinverse-guided_2023}
Song, J., Vahdat, A., Mardani, M., and Kautz, J.
\newblock Pseudoinverse-{Guided} {Diffusion} {Models} for {Inverse} {Problems}.
\newblock In \emph{International Conference on Learning Representations
  (ICLR)}, 2023{\natexlab{a}}.

\bibitem[Song et~al.(2023{\natexlab{b}})Song, Zhang, Yin, Mardani, Liu, Kautz,
  Chen, and Vahdat]{song_loss-guided_2023}
Song, J., Zhang, Q., Yin, H., Mardani, M., Liu, M.-Y., Kautz, J., Chen, Y., and
  Vahdat, A.
\newblock Loss-guided diffusion models for plug-and-play controllable
  generation.
\newblock In \emph{Proceedings of the 40th International Conference on Machine
  Learning (ICML)}, volume 202 of \emph{Proceedings of {Machine} {Learning}
  {Research}}, pp.\  32483--32498. PMLR, 2023{\natexlab{b}}.

\bibitem[Song et~al.(2021{\natexlab{a}})Song, Durkan, Murray, and
  Ermon]{song_likelihood_2021}
Song, Y., Durkan, C., Murray, I., and Ermon, S.
\newblock Maximum likelihood training of score-based diffusion models.
\newblock In \emph{Advances in Neural Information Processing Systems 34
  (NeurIPS)}, pp.\  1415--1428, 2021{\natexlab{a}}.

\bibitem[Song et~al.(2021{\natexlab{b}})Song, Sohl-Dickstein, Kingma, Kumar,
  Ermon, and Poole]{song_score-based_2021}
Song, Y., Sohl-Dickstein, J., Kingma, D.~P., Kumar, A., Ermon, S., and Poole,
  B.
\newblock Score-{Based} {Generative} {Modeling} through {Stochastic}
  {Differential} {Equations}.
\newblock In \emph{International Conference on Learning Representations
  (ICLR)}, 2021{\natexlab{b}}.

\bibitem[Tang et~al.(2024)Tang, Peng, Tang, Hong, Wang, and
  Chang]{tang2024tuning}
Tang, Z., Peng, J., Tang, J., Hong, M., Wang, F., and Chang, T.-H.
\newblock Tuning-free alignment of diffusion models with direct noise
  optimization.
\newblock In \emph{ICML 2024 Workshop on Structured Probabilistic Inference \&
  Generative Modeling}, 2024.

\bibitem[Tomczak(2022)]{tomczak2022deep}
Tomczak, J.~M.
\newblock \emph{Deep Generative Modeling}.
\newblock Springer Cham, 1 edition, February 2022.

\bibitem[Valencia et~al.(2025)Valencia, Pfaff, and
  Thuerey]{valencia2025learning}
Valencia, M.~L., Pfaff, T., and Thuerey, N.
\newblock Learning distributions of complex fluid simulations with diffusion
  graph networks.
\newblock In \emph{The Thirteenth International Conference on Learning
  Representations}, 2025.

\bibitem[Vincent(2011)]{vincent_connection_2011}
Vincent, P.
\newblock A connection between score matching and denoising autoencoders.
\newblock \emph{Neural Computation}, 23\penalty0 (7):\penalty0 1661--1674,
  2011.

\bibitem[Wang et~al.(2023{\natexlab{a}})Wang, Yu, and Zhang]{wangzero}
Wang, Y., Yu, J., and Zhang, J.
\newblock Zero-shot image restoration using denoising diffusion null-space
  model.
\newblock In \emph{International Conference on Learning Representations
  (ICLR)}, 2023{\natexlab{a}}.

\bibitem[Wang et~al.(2023{\natexlab{b}})Wang, Lu, Wang, Bao, Li, Su, and
  Zhu]{wang2023prolificdreamer}
Wang, Z., Lu, C., Wang, Y., Bao, F., Li, C., Su, H., and Zhu, J.
\newblock Prolificdreamer: High-fidelity and diverse text-to-3d generation with
  variational score distillation.
\newblock In \emph{Advances in Neural Information Processing Systems 36
  (NeurIPS)}, pp.\  8406--8441, 2023{\natexlab{b}}.

\bibitem[Woessner \& Poeter(2020)Woessner and
  Poeter]{woessner2020hydrogeologic_ch4}
Woessner, W.~W. and Poeter, E.~P.
\newblock Chapter 4: Darcy’s law, head, gradient and hydraulic conductivity.
\newblock In \emph{Hydrogeologic Properties of Earth Materials and Principles
  of Groundwater Flow}. The Groundwater Project, Guelph, Ontario, Canada, 2020.

\bibitem[Yao et~al.(2025)Yao, Mammadov, Berner, Kerrigan, Ye, Azizzadenesheli,
  and Anandkumar]{yao2025guided}
Yao, J., Mammadov, A., Berner, J., Kerrigan, G., Ye, J.~C., Azizzadenesheli,
  K., and Anandkumar, A.
\newblock Guided diffusion sampling on function spaces with applications to
  pdes.
\newblock \emph{arXiv preprint arXiv:2505.17004}, 2025.

\bibitem[Zhang et~al.(2023)Zhang, Rao, and Agrawala]{zhang2023adding}
Zhang, L., Rao, A., and Agrawala, M.
\newblock Adding conditional control to text-to-image diffusion models.
\newblock In \emph{Proceedings of the IEEE/CVF international conference on
  computer vision}, pp.\  3836--3847, 2023.

\bibitem[Zhu et~al.(2023)Zhu, Zhang, Liang, Cao, Wen, Timofte, and
  Gool]{zhu_plugnplayrestore_2023}
Zhu, Y., Zhang, K., Liang, J., Cao, J., Wen, B., Timofte, R., and Gool, L.~V.
\newblock Denoising diffusion models for plug-and-play image restoration.
\newblock In \emph{Proceedings of the IEEE/CVF Conference on Computer Vision
  and Pattern Recognition Workshops (CVPRW)}, pp.\  1219--1229, 2023.

\bibitem[Zou et~al.(2024)Zou, Meng, and Karniadakis]{zou_correctingpinn_2024}
Zou, Z., Meng, X., and Karniadakis, G.~E.
\newblock Correcting model misspecification in physics-informed neural networks
  ({PINN}s).
\newblock \emph{Journal of Computational Physics}, 505, 2024.

\end{thebibliography}
